\documentclass[11pt]{article}
\usepackage{amsthm}
\usepackage{amsmath}
\usepackage{amssymb}
\usepackage{gensymb}
\usepackage{amsfonts}
\usepackage{svg}
\usepackage{graphicx}
\usepackage[ruled]{algorithm}
\usepackage[noend]{algpseudocode}
\usepackage{enumerate}
\usepackage{natbib}
\usepackage{todonotes}
\usepackage{mathtools}
\usepackage{subcaption}
\mathtoolsset{showonlyrefs}
\usepackage{comment}
\usepackage{tikz-cd}
\usepackage{caption,subcaption}
\usepackage{soul}
\usepackage[makeroom]{cancel}
\usepackage{setspace}
\usepackage{wrapfig}
\usepackage[margin=1in]{geometry}

\definecolor{darkblue}{rgb}{0.0, 0.0, 0.8}
\usepackage[colorlinks=true, allcolors=darkblue]{hyperref}

\usepackage{array}
\newcommand{\PreserveBackslash}[1]{\let\temp=\\#1\let\\=\temp}
\newcolumntype{C}[1]{>{\PreserveBackslash\centering}p{#1}}
\newcolumntype{R}[1]{>{\PreserveBackslash\raggedleft}p{#1}}
\newcolumntype{L}[1]{>{\PreserveBackslash\raggedright}p{#1}}

\newtheorem{proposition}{Proposition}
\newtheorem{thm}{Theorem}
\newtheorem{lemma}{Lemma}

\newtheorem{remark}{Remark}
\newtheorem{definition}{Definition}
\newtheorem{corollary}{Corollary}

\newcommand{\R}{\mathbb{R}}

\newcommand{\vol}{\operatorname{vol}}

\title{SVarM: Linear Support Varifold Machines for Classification and Regression on Geometric Data}
\author{Emmanuel Hartman, Nicolas Charon \\
  Department of Mathematics\\
  University of Houston\\
}

\begin{document}
\maketitle
\begin{abstract}
Despite progress in the rapidly developing field of geometric deep learning, performing statistical analysis on geometric data—where each observation is a shape such as a curve, graph, or surface—remains challenging due to the non-Euclidean nature of shape spaces, which are defined as equivalence classes under invariance groups. Building machine learning frameworks that incorporate such invariances, notably to shape parametrization, is often crucial to ensure generalizability of the trained models to new observations. This work proposes \textit{SVarM} to exploit varifold representations of shapes as measures and their duality with test functions $h:\R^n \times S^{n-1} \rightarrow \R$. This method provides a general framework akin to linear support vector machines but operating instead over the infinite-dimensional space of varifolds. We develop classification and regression models on shape datasets by introducing a neural network-based representation of the trainable test function $h$. This approach demonstrates strong performance and robustness across various shape graph and surface datasets, achieving results comparable to state-of-the-art methods while significantly reducing the number of trainable parameters.
\end{abstract}

\section{Introduction}
Shape analysis, or geometric data science, is a field dedicated to building statistical and machine learning methods able to retrieve and analyze the morphological variability in geometric structures. This is a particularly central problem in applications such as computer vision or biomedical imaging where observations often come as segmented curves, surfaces, densities or other types of complex geometric data. Various approaches have been proposed, including Riemannian and elastic shape space models \citep{grenander1996elements,michor2007overview,jermyn2012elastic,younes2019shapes}, topology based methods \citep{ghrist2008barcodes,carlsson2009topology,edelsbrunner2017persistent}, and metric/functional matching frameworks \citep{bronstein2006generalized,memoli2011gromov,ovsjanikov2012functional}. These different methods have proved quite successful in tackling problems such as pairwise comparison, regression, classification or clustering for data sets of shapes. 
%However, with the constant advances in acquisition protocols and the explosion in the size and resolution of data sets that followed, many such methods do not always scale well to recent applications that may involve databases with up to tens of thousands of subjects, each made of potentially hundreds of thousands of vertices. 
However, acquisition protocols have advanced significantly, leading to much larger and higher-resolution datasets.
As a result, many existing methods do not scale well to modern applications involving tens of thousands of subjects, each with potentially hundreds of thousands of vertices.
In view of the rapid development of new machine learning paradigms, in particular neural network models, and their impressive achievements in image processing and analysis tasks, one can reasonably expect similar tools to be able to address those challenges on geometric data. Yet, the very particular and intricate nature of shape spaces poses unique challenges when it comes to designing robust neural network models for shape analysis tasks. The importance of this problem within machine learning was recently recognized through the emergence of the research area known \textit{geometric deep learning}, see the surveys \citep{bronstein2017geometric,cao2020comprehensive,gerken2023geometric} and references therein for the current state-of-the-art.    

In contrast with many works in geometric deep learning that focus on developing neural network architectures (in particular convolutional models) defined directly on e.g. surface meshes \citep{kostrikov2018surface,jiang2019semi,hanocka2019meshcnn}, the approach we propose in this paper instead relies on the embedding of shapes into a certain measure space. Specifically, our learning framework leverages the concept of \textit{varifolds} from geometric measure theory \citep{Almgren66,Allard72} along with the natural representation of a shape $q \in \mathcal{S}$---such as a curve or a surface in $\R^n$---as its corresponding varifold $\mu_q \in \mathcal{V}$. The mapping $q\mapsto \mu_q$ is invariant under reparameterization, thereby providing a (nonlinear) embedding of the shape space $\mathcal{S}$ into $\mathcal{V}$, the space of measures (or distributions) on $\R^n \times S^{n-1}$. This enables us to instead build the learning framework on the space of varifolds by exploiting the duality of these measures with continuous test functions on $\R^n \times S^{n-1}$.We simply propose to rely on affine functionals on $\mathcal{V}$ of the form $\mu \mapsto \langle \mu, h\rangle + \beta$. Here, $h \in C_0(\R^n \times S^{n-1})$ and $\beta \in \R$ become trainable model parameters that are to be estimated for a given statistical task. Due to its analogy with support vector machines (but considered in the infinite-dimensional varifold space), we coin our model \textbf{S}upport \textbf{Var}ifold \textbf{M}achines (SVarM) and derive formulations of the shape regression and classification problems based on this setting. This overall idea also shares certain similarities with recent works such as \citep{bhadra2025scalar} which introduces a regression model for function data based on the square root velocity function (SRVF) representation. Our use of varifolds provides, arguably, a more convenient setting when dealing with geometric instead of functional data. Importantly, our approach allows for a principled handling of complex shapes without requiring prior regularization or registration. Additionally, by focusing on approximating continuous functions in a low-dimensional space, the model remains lightweight, requiring significantly fewer trainable parameters than traditional models operating on raw geometric data as inputs.

The presentation and contributions of the paper are organized into three main sections. Section \ref{sec:maths_foundations} introduces the mathematical settings of shape spaces and of varifolds. We then proceed by deriving a few theoretical results related to the approximation properties of SVarM for regression and classification as well as its stability with respect to geometric perturbations. Section \ref{sec:netw_architecture} details our proposed neural network architecture, which approximates the trainable test function $h$ of the model. Finally, Section \ref{sec:numerical} demonstrates the practical effectiveness of the approach on various regression and classification problems for real surface and shape graph data sets. 

% The presentation is structured into three main sections:

% \begin{enumerate}
% \item \textbf{Mathematical Foundations and Approximation Results:} We begin by introducing the theoretical underpinnings of our framework. Drawing on classical results from functional analysis, we present a series of approximation theorems. Specifically:
% \begin{itemize}
%     \item In \textbf{Theorem~\ref{thm:affine_approx}}, we establish the existence of solutions for the approximation of affine functions on the space of Varifolds.
%     \item In \textbf{Theorem~\ref{thm:separating_convex}}, we extend this result to show that our method can approximate functions that separate disjoint convex subsets within this space, a key capability for classification tasks.
% \end{itemize}

% \item \textbf{Model Architecture:} We then detail our proposed neural network architecture, which approximates an element of the pre-dual of the space Varifolds. By focusing on approximating continuous functions in a low-dimensional space, the model remains lightweight, requiring significantly fewer trainable parameters than traditional models operating on raw geometric data as inputs.

% \item \textbf{Numerical Experiments:} Finally, we present a set of experiments that demonstrate the practical effectiveness of our approach. These results emphasize the model's ability to generalize across complex geometric structures, and validate the theoretical efficiency gains---especially the strong performance with a reduced number of parameters.
% \end{enumerate}

\begin{figure}[h!]
\centering
  \includegraphics[width=.9\textwidth]{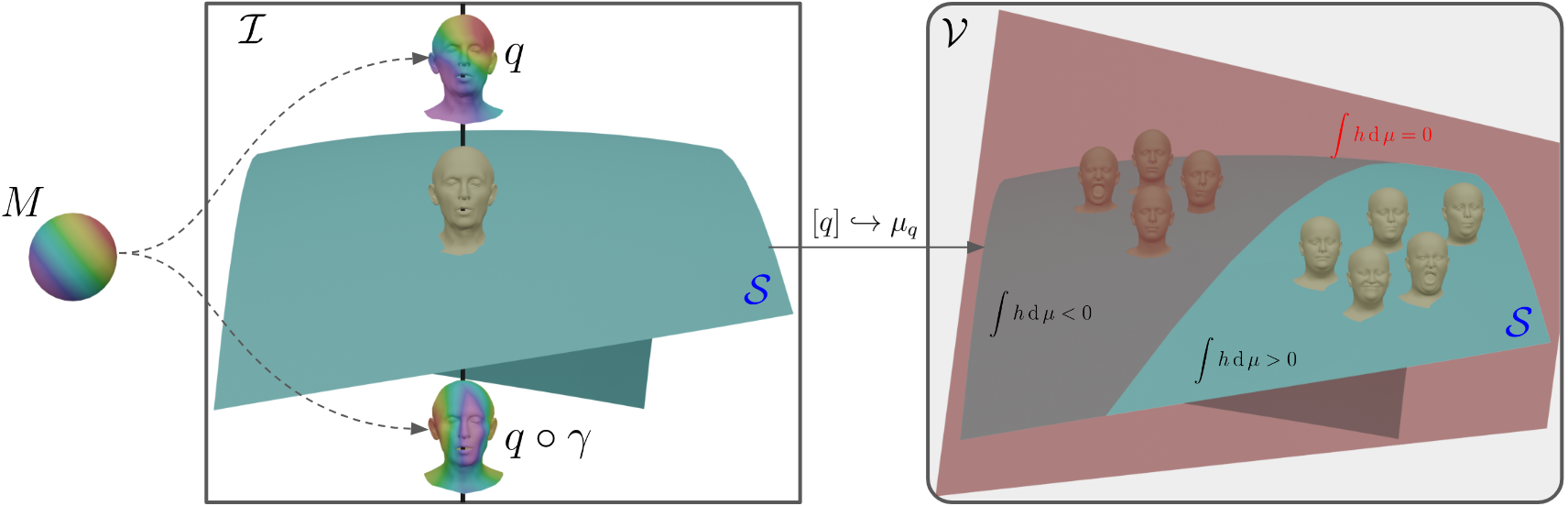} 
  \caption{Representation of the quotient shape space $\mathcal{S}$ (left) and its embedding in varifold space $\mathcal{V}$ (right). A linear form $\mu \mapsto \langle \mu , h\rangle$ used by SVarM is visualized as the red hyperplane in $\mathcal{V}$.}\label{fig:SVarM:graph_bastract}
\end{figure}

\section{Mathematical foundations and results}
\label{sec:maths_foundations}
The concept of "shape" encompasses a wide variety of geometric structures (see e.g. \citep{younes2019shapes,bauer2014overview,srivastava2016functional} for general overviews on shape analysis). Although the framework we introduce can in principle apply to any type of shapes that can be viewed as measures or varifolds, our focus here will be on open/closed curves embedded in the plane or 3D space, surfaces in $\R^3$, and by extension structures known as shape graphs, see \citep{sukurdeep2022new,bal2024statistical}. The challenging aspect of dealing with such objects is that there is generally not a simple nor canonical way to represent them in a common Euclidean or even functional space. Indeed, a curve or surface is usually defined by a parametrization function $q:M \rightarrow \R^n$, where $n=2$ or $3$ and $M$ is the parameter manifold (for instance an interval or $S^1$ for open and closed curves, a disk or $S^2$ for open/closed surfaces$\ldots$).
\subsection{Shape Space}  
In this paper, we consider oriented rectifiable sets parameterized by bi-Lipschitz regular embeddings of of a compact, oriented, Riemannian manifold with or without boundary, $M$ to $\R^n$. We denote such an embedding by $q\in \mathcal{I} \doteq \operatorname{Emb}_{\text{Lip}}(M,\R^n)$. However, $q$ itself does not uniquely identify the shape $q(M)\subseteq \R^n$. The action of any reparametrization $q \circ \gamma$, where $\gamma \in \mathcal{D} = \operatorname{Diff}_{\text{Lip}}(M)$ is an orientation preserving bi-Lipschitz homeomorphism of $M$, leaves the shape $q(M)$ invariant. In other words, the actual shape associated to $q$ is in essence the equivalence class $[q]=\{q \circ \gamma: \ \gamma \in \operatorname{Diff}_{\text{Lip}}(M)\}$, and the shape space $\mathcal{S}$ can be thus identified with the quotient $\mathcal{S} = \mathcal{I}/\mathcal{D}$, see Figure \ref{fig:SVarM:graph_bastract} for an illustration. Note that quotienting out other shape invariances, such as to rigid motions of $\R^n$, can also be relevant in certain situations, but this work focuses on creating a parametrization-invariant model, as rigid alignment can usually be applied in a separate step. In fact, Section \ref{ssec:regression_exp} and Section \ref{ssec:full_rotation_exp} illustrate how SVarM can itself be used for that very task.

We want to add that the above shape space setting also comprises discrete objects, e.g. polygonal curves and triangulated surfaces. This setting also includes discrete objects such as polygonal curves and triangulated surfaces. In these cases, $M$ is subdivided into a 1- or 2- dimensional mesh, and the parametrization q is a Lipschitz piecewise affine map over that mesh. Consequently, $q(M)$ can be viewed as a union of adjacent cells in $\R^n$. The issue of parametrization invariance still remains very relevant in this situation. Namely, a proper learning model should exhibit some degree of robustness to operations such as mesh subdivision, simplification or resampling.

\subsection{Varifolds} 
\label{ssec:varifolds}
The concept of varifold goes back to early works in geometric measure theory, notably \citep{Almgren66,Allard72}, where it provided a natural and convenient setting to reframe geometric variational problems into problems on measure spaces. In more recent years, multiple works in applied mathematics have leveraged varifold representations in the context of shape registration \citep{charon2013varifold,kaltenmark2017general,lee2024neural,stouffer2024cross}, computational geometry \citep{hsieh2021metrics,buet2022weak,paul2024sparse}, or computer vision \citep{pierson20223d,besnier2023toward,hartman2025basis}. One of the fundamental features explaining the interest in varifolds for these different problems is that it provides a representation of shapes, both continuous and discrete, that is independent of parametrization/sampling, as we shall explain below. This suggests that it can be well-suited in the design of learning models on shape spaces. 

\begin{definition}
    A varifold is a positive Radon measure on $\R^n \times S^{n-1}$ where $S^{n-1}$ denotes the unit sphere of $\R^n$. We shall write, in short, $\mathcal{V} \doteq \mathcal{M}^+(\R^n \times S^{n-1})$ the space of oriented varifolds. 
\end{definition}
To be exact, the above definition corresponds to dimension/codimension 1 oriented varifolds of $\R^n$. A more general construction can be obtained by replacing $S^{n-1}$ with Grassmannians, c.f. \citep{Allard72} or \citep{charon2013varifold}. We will denote $\text{supp}(\mu) \subseteq \R^n \times S^{n-1}$ the \textit{support} of a varifold $\mu \in \mathcal{V}$. Furthermore, by the usual Riesz-Markov-Kakutani representation theorem, any varifold $\mu$ can be viewed as an element of the dual to $C_0(\R^n \times S^{n-1})$, the space of continuous test functions on $\R^n \times S^{n-1}$ vanishing at infinity and we shall use the notation $\langle \mu , h\rangle = \int_{\R^n \times S^{n-1}} h(x,v) d\mu(x,v)$ for the duality bracket between $\mathcal{V}$ and $C_0(\R^n \times S^{n-1})$. In particular, a Dirac varifold $\delta_{(x_0,v_0)}$ for some $(x_0,v_0) \in \R^n \times S^{n-1}$ is such that $\langle \delta_{(x_0,v_0)} , h\rangle =h(x_0,v_0)$.  

Now, going back to shape spaces, any Lipschitz embedded curve or surface $q:M \rightarrow \R^n$ (where $M$ is of dimension $1$ or $n-1$) can be canonically mapped to a varifold $\mu_q \in \mathcal{V}$ that is defined as the pushforward measure $\mu_q : =(q,v_q)_*\vol_q$, where $v_q$ denotes the unit tangent/normal vector field and $\vol_q$ is the induced length/area measure on $M$. More explicitly, for any $h\in C_0(\R^n\times S^{n-1})$,
\begin{equation}
    \langle \mu_q , h \rangle = \int_{\R^n\times S^{n-1}} h(x,v) \operatorname{d}\mu_q(x,v)= \int_M h(q(m),v_q(m)) \operatorname{d} \vol_q(m).
\end{equation}
Intuitively, $\mu_q$ represents the distribution of tangent/normal vectors at the different points of the shape $q(M)$ with local density given by the length/area measure along $q(M)$. A key property of the mapping $q \mapsto \mu_q$ is that it only depends on the oriented image set $q(M)$ and not on parametrizations:
\begin{thm}
\label{prop:varifold_invariance}
    For any $q\in \mathcal{I}$ and $\gamma\in\mathcal{D}$, it holds that $ \mu_{q\circ\gamma}=\mu_q$. Consequently, one obtains a well-defined map $[q] \in \mathcal{S} \mapsto \mu_q \in \mathcal{V}$ which, in addition, is injective. 
\end{thm}
\begin{proof}
This result and its proof are similar to that of previous works such as \citep{charon2013varifold,bauer2019relaxed}, albeit in a slightly different setting. We recap the main arguments, most of which rely on known results from geometric measure theory. Introducing all the related definitions and concepts in detail is out of the scope of this paper and we will thus refer the reader to \citep{simon2014introduction} for a precise exposition. 

First, we need to show that the $\text{Var}$ map is well-defined in the sense that $\mu_q$ is independent of the parametrization. Denoting the embedded shape $S=q(M)$, $S$ is a rectifiable subset of $\R^n$ of dimension $d$ ($d=1,2$ being the dimension of the parameter space $M$). Then the area formula (Theorem 3.3 in \citep{simon2014introduction}) implies in this case that:
\begin{align}
\label{eq:var_rect_set}
    \langle \mu_{q}, h\rangle = \int_M h(q(m), v_{q}(m)) d \, \text{vol}_{q}(m) = \int_{S} h(x,v_S(x)) d\mathcal{H}^{d}(x) 
\end{align}
where $\mathcal{H}^d$ denotes the $d$-dimensional Hausdorff measure of $\R^n$, and $v_S(x)$ is the tangent/normal vector to $S$ at $x$ which is defined for $\mathcal{H}^{d}$-a.a. $x \in S$. This shows that $\langle \mu_q,h\rangle$ only depends on $S=q(M)$ and not the parametrization $q$ itself.

Next, we need to prove the injectivity of the varifold map on $\mathcal{S}$. Let $[q]$, $[q']$ be two distinct shapes of $\mathcal{S}$. Then $S=q(M)$ and $S'=q'(M)$ are two rectifiable subsets of dimension $d$ of $\R^n$. Let us first consider the case in which those two sets are distinct. Without loss of generality, we may then assume that $\mathcal{H}^d(S\setminus S') >0$. By continuity of $q$ and $q'$ and the compactness of $M$, one can deduce that there exist $x_0 \in S$, $r>0$ and $\epsilon>0$ such that $\text{dist}((B(x_0,r),S')> \epsilon >0$. We can thus construct a nonnegative test function $h \in C_0(\R^n \times S^{n-1})$ such that $h(x,v)=1$ for any $x \in B(x_0,r)$ and $h(x,v)=0$ for $\|x-x_0\|\geq r+\epsilon$. From \eqref{eq:var_rect_set}, we obtain that:
\begin{equation*}
  \langle \mu_{q}, h\rangle = \int_{S} h(x,v_S(x)) d\mathcal{H}^{d}(x) \geq \int_{B(x_0,r)\cap S} h(x,v_S(x)) d\mathcal{H}^{d}(x) = \mathcal{H}^d(B(x_0,r)\cap S)>0
\end{equation*}
while
\begin{equation*}
  \langle \mu_{q'}, h\rangle = \int_{S'} h(x,v_S(x)) d\mathcal{H}^{d}(x)=0
\end{equation*}
since the support of $h$ does not intersect $S'$ by construction. It follows that $\langle \mu_q,h \rangle \neq \langle \mu_{q'},h \rangle$ and thus the two varifolds $\mu_q$ and $\mu_{q'}$ are distinct in $\mathcal{V}$. On the other hand, if the two sets $S$ and $S'$ coincide, then we have a well-defined bi-Lipschitz homeomorphism $\psi: M \rightarrow M$ with $\psi = q^{-1} \circ q'$ and where $q^{-1}$ is the inverse of $q$ on the image set $q(M)=q'(M)$. By definition, $q' = q \circ \psi$ and since we assumed that $[q]\neq[q']$, $\psi$ cannot be orientation-preserving. This implies that one can find a subset $B \subseteq M$ with $\mathcal{H}^d(q'(B)) \neq 0$ such that for any $x \in B$, we have $v_{q\circ \psi}(x) = - v_{q'}(x)$. Since $\mathcal{H}^d(q'(B)) \neq 0$, there exists $\phi:\R^3\to \R$ such that $\phi(x)=0$ for $x\in q'(M - B)$ and $u:=\int_B \phi(q'(x)) v_{q'}(x)\, d\operatorname{vol}_{q'}(x)\neq 0$. Therefore, 
\begin{align*}
&\int_B \phi(q') (v_{q'}\cdot u) \, d\operatorname{vol}_{q'} = |u|^2 > 0, \ \text{and} \\
&\int_B \phi(q\circ \psi) (v_{q\circ\psi}\cdot u)\,d\operatorname{vol}_{q\circ\psi} = \int_B \phi(q')(-v_{q'}\cdot u)\, d\operatorname{vol}_{q'}< 0
\end{align*}
It follows that for $h\in C_0(\R^n \times S^{n-1})$ defined by $h(x,v) = \phi(x) (u \cdot v)$, we have $\langle \mu_q, h \rangle \neq \langle \mu_{q'}, h \rangle$ and so $\mu_q \neq \mu_{q'}$.
\end{proof}
This property will be fundamental for what follows, as it allows us to embed any shape of $\mathcal{S}$ into $\mathcal{V}$ (see Figure \ref{fig:SVarM:graph_bastract}) and exploit the dual nature of $\mathcal{V}$.

Another advantage of the varifold framework is that it applies very naturally to the discrete geometric setting as well. Specifically, when $q$ is a piecewise affine embedding defined over a polyhedral mesh on $M$, the shape $q(M)$ is made of a reunion of polyhedral cells $\{F_i\}$ in $\R^n$ (segments or triangles) which can each be approximated by a single weighted Dirac varifold $w_i\delta_{(x_i,v_i)}$, with $x_i \in \R^n$, $v_i \in S^{n-1}$ and $w_i>0$ being respectively the center of mass, unit tangent/normal vector and length/area of $F_i$. This leads to a discrete varifold approximation $\mu_q \approx \sum_{i} w_i \delta_{(x_i,v_i)}$ of $q$. This type of approximations in $\mathcal{V}$ has been routinely used and analyzed in works such as \citep{kaltenmark2017general,buet2022weak,paul2024sparse}, to which we refer for more details.
% \begin{proof}
%     Let $h\in C_0(\R^3\times S^2,\R)$ and observe that by a change of variables
%     \begin{equation*}
%         \int_M h(q\circ \gamma ,n_{q\circ \gamma}) \operatorname{d} \vol_{q\circ \gamma} =  \int_M h(q\circ \gamma ,n_{q}\circ \gamma)  |D\gamma|\operatorname{d} \vol_{q}=\int_M h(q,n_q) \operatorname{d} \vol_q.
%     \end{equation*}
% \end{proof}

\subsection{Theoretical Properties of SVarM}
\label{sec:theo_res}
As stated in the introduction, the core of the SVarM framework is to represent affine weak-* continuous forms on $\mathcal{V}$ which can then trained as regressors or classifiers on shape spaces via the identification given by Theorem \ref{prop:varifold_invariance}. Before we delve into the training process itself in Section \ref{sec:SVarM_training}, we shall first discuss the precise mathematical class of affine functionals that will be considered in SVarM. Moreover, we will establish a few theoretical properties of those functionals that are directly relevant for the shape regression or classification problems considered later on. 

\subsubsection{Existence of Regression Functions and Separators}
Let us first recall a few definitions related to the different topologies on varifolds. The space of signed measures $\mathcal{M}(\R^n \times S^{n-1})$ can be identified with the topological dual to the space of test functions $C_0(\R^n \times S^{n-1})$ equipped with the $\|\cdot\|_\infty$ norm. As such, different topologies are usually considered on $\mathcal{M}(\R^n \times S^{n-1})$ and by restriction on $\mathcal{V}$. The first of which is the dual metric topology, i.e. the topology induced by the \emph{total variation} norm: 
\begin{equation}
    \label{eq:TV_norm_def}
    \|\mu\|_{TV}=\sup_{\|h\|_{\infty}\leq 1} |\langle \mu,h \rangle| = |\mu|(\R^n \times S^{n-1})
\end{equation}
In particular, for a varifold $\mu_q$ associated to $q \in \mathcal{I}$, one can see from \eqref{eq:var_rect_set} that $\|\mu_q\|_{TV}$ is the total length/area of the shape $q(M)$. The TV norm induces a rather strong topology on varifolds: for instance, the TV norm between two Dirac masses $\delta_{(x_0,v_0)}$ and $\delta_{(x_1,v_1)}$ is $2$ whenever $(x_0,v_0) \neq (x_1,v_1)$ and $0$ when the position and direction coincide. It is often more natural to consider the \emph{weak-* topology} on $\mathcal{M}(\R^n \times S^{n-1})$. By definition, a sequence $(\mu_i)$ weak-* converges to $\mu$, denoted $\mu_i \buildrel\ast\over\rightharpoonup \mu$, when $\langle \mu_i,h \rangle \xrightarrow[i\rightarrow \infty]{} \langle \mu, h \rangle$ for all $h \in C_0(\R^n \times S^{n-1})$. Note that the weak-* convergence can be also locally metrized by a variety of usual metrics on $\mathcal{M}(\R^n \times S^{n-1})$ such as the bounded Lipschitz norm or the Wasserstein metrics of optimal transport (for normalized measures), c.f. \citep{villani2008optimal} Chap. 6.

As a consequence of standard results on Banach spaces and weak-* topology, specifically Proposition 3.14 in \citep{brezis2011functional}, any linear function  $\mathcal{V} \rightarrow \R$ that is continuous with respect to the weak-* convergence is necessarily of the form $\mu \mapsto \langle \mu, h \rangle $ for some test function $h \in C_0(\R^n \times S^{n-1})$. It follows that the set of affine weak-* continuous functions $\varphi:\mathcal{V} \rightarrow \R$ are of the form $\varphi(\mu) = \langle \mu,h\rangle +\beta$ for some $h \in C_0(\R^n \times S^{n-1})$ and $\beta \in \R$. Thus we have the following:
\begin{proposition}
\label{prop:SVarM_affine_maps}
    The set of functionals modeled by SVarM are exactly the affine weak-* continuous functionals on $\mathcal{V}$. 
\end{proposition}
However, this does not hold anymore when the continuity is with respect to the stronger TV metric topology due to the lack of reflexiveness of $\mathcal{V}$. In that case, one can only recover the following approximation bound over any finite set of varifolds:
\begin{proposition}
    Suppose $\varphi:\mathcal{V} \to \R$ is an affine function and continuous with respect to the TV norm topology and fix $C\subseteq \mathcal{V}$ any finite subset of the space of varifolds. For every $\epsilon>0$, there exists $h\in C_0(\R^n\times S^{n-1},\R)$ and $\beta \in \R$ such that 
    $$\sup_{\mu\in C}\left|\varphi(\mu) - (\langle \mu, h\rangle + \beta)\right|<\epsilon.$$
\end{proposition}
\begin{proof}
    Let us first consider the case of linear forms i.e. let $\varphi: \mathcal{M}(\R^n\times S^{n-1}) \to \R$ be a continuous linear functional, in other words $\varphi$ belongs to the bidual space $C_0(\R^n\times S^{n-1},\R)^{**}$.  By Goldstine theorem \citep{rudin1991functional}, $C_0(\R^n\times S^{n-1},\R)$ is dense in its bi-dual with respect to the weak-* topology on $C_0(\R^n\times S^{n-1},\R)^{**}$. Specifically, for the canonical embedding $J$ defined by:
    \begin{align}
        J:C_0(\R^n\times S^{n-1},\R)&\to C_0(\R^n\times S^{n-1},\R)^{**} \\
        h&\mapsto (\mu \in\mathcal{M}(\R^n\times S^{n-1}) \mapsto \langle \mu , h \rangle)
    \end{align} 
    we have $J(C_0(\R^n\times S^{n-1},\R))$ dense in $C_0(\R^n\times S^{n-1},\R)^{**}$. Consequently, we can find a sequence $(h_i)\subseteq C_0(\R^n\times S^{n-1},\R)$ so that $J(h_i)(\mu)=\langle \mu , h_i \rangle \to \varphi(\mu)$ as $i \rightarrow \infty$ for all $\mu\in\mathcal{M}(\R^n\times S^{n-1})$. Now , since $C$ is assumed to be finite, choosing $i$ large enough so that $|\varphi(\mu)-\langle \mu , h_i \rangle| < \epsilon$ for all $\mu \in C$ leads to the desired results. The extension to affine continuous functionals is then straightforward by incorporating the extra bias term $\beta \in \R$.
\end{proof}

For classification problems, we are interested in determining when two given sets of varifolds (or shapes) can be separated by SVarM. 
\begin{definition}
    Two disjoint sets $C_1,C_2 \subseteq \mathcal{V}$ are strictly separated by some test function $h$ when:
\begin{equation}
\label{eq:def_separated}
\sup_{\mu\in C_1} \langle \mu,h\rangle \,<\, \inf_{\nu\in C_2} \langle \nu,h\rangle
%\sup_{\mu\in C_1}\int_{\R^3 \times S^2}h(x,v)\operatorname{d}\mu\,<\, \inf_{\nu\in C_2}\int_{\R^3 \times S^2}h(x,v)\operatorname{d}\nu.
\end{equation} 
\end{definition}
Note that, by reintroducing the extra bias parameter $\beta \in \R$, the above is also equivalent to $\sup_{\mu \in C_1} \langle \mu,h\rangle + \beta < 0 < \inf{\nu \in C_2}\langle \nu,h\rangle +\beta$, which is analogous to the separation by a hyperplane property in classical SVM. A first generic result, that follows from the Hahn-Banach separation theorem is:  
\begin{thm}\label{thm:separating_convex}
    Suppose $C_1, C_2 \subseteq \mathcal{V}$ are convex weak-* compact subsets such that $C_1\cap C_2 = \emptyset$.
    Then there exists a test function $h\in C_0(\R^n\times S^{n-1},\R)$ that separates $C_1$ and $C_2$ in the sense of \eqref{eq:def_separated}. 
\end{thm}
\begin{proof}
Let $C_1, C_2 \subseteq \mathcal{M}(\R^n\times S^{n-1})$ be convex weak-* compact subsets such that $C_1\cap C_2 = \emptyset$. Since the space $\mathcal{M}(\R^n\times S^{n-1})$ with the weak-* topology is a locally convex space, the Hahn-Banach separation theorem (Theorem 1.7 in \citep{brezis2011functional}) applies and there exists a weak-* continuous linear form $\varphi: \mathcal{V} \rightarrow \R$ such that: 
\begin{equation*}
    \sup_{\mu\in C_1} \varphi(\mu)< \inf_{\nu\in C_2} \varphi(\nu).
\end{equation*}
As already noted above, $\varphi$ must be of the form $\varphi(\mu) = \langle \mu, h \rangle$ for some test function $h \in C_0(\R^n\times S^{n-1},\R)$, which leads to the result.
\end{proof}
The weak-* compactness assumption is relatively natural as, by the Banach-Alaoglu theorem, closed sets of varifolds with bounded total variation are in particular weak-* compact. However, the condition of convexity in $\mathcal{V}$ in Theorem \ref{thm:separating_convex} is rather abstract and not immediate to connect to concrete situations of interest. We will thus look to derive more specific conditions for certain families of measure sets $C_1$ and $C_2$. First of all, we can focus on the situation where $C_1$ and $C_2$ are both finite subsets of $\mathcal{V}$, which is natural when considering classification and regression problems in practice. In that case, $C_1$ and $C_2$ are automatically weak-* compact. We can then state the following simpler condition:
\begin{thm}\label{thm:separating_support}
    Let $C_1, C_2 \subseteq \mathcal{V}$ be disjoint finite sets of varifolds which satisfy the property that either 1) for each $\nu \in C_2$, $\text{supp}(\nu)\not\subseteq \bigcup_{\mu\in C_1} \text{supp}(\mu)$ or 2) for each $\mu \in C_1$, $\text{supp}(\mu)\not\subseteq \bigcup_{\nu\in C_2} \text{supp}(\nu)$. Then there exists $h\in C_0(\R^n\times S^{n-1},\R)$ that separates $C_1$ and $C_2$.
\end{thm}
\begin{proof}
Without loss of generality, let us assume that $C_1, C_2 \subseteq \mathcal{S}$ are two finite sets in $\mathcal{V}$ with $C_1=\{\mu_1,\ldots,\mu_p\}$ and $C_2=\{\nu_1,\ldots,\nu_q\}$, such that for each $\nu \in C_2$, $\text{supp}(\nu)\not\subseteq \bigcup_{\mu\in C_1} \text{supp}(\mu)$. Let us denote by $\text{conv}(C_1)$ and $\text{conv}(C_2)$ the convex hulls of the two sets $C_1$ and $C_2$ in $\mathcal{V}$. By contradiction, if $\text{conv}(C_1) \cap \text{conv}(C_2) \neq \emptyset$, there exist $(\alpha_1,\ldots,\alpha_p)$ and $(\beta_1,\ldots,\beta_q)$ in the probability simplexes of $\R^p_+$ and $\R_+^q$ respectively such that:   
\begin{equation*}
    \sum_{i=1}^p \alpha_i \mu_i = \sum_{j=1}^{q} \beta_j \nu_j \ \implies \ \operatorname{supp}\left(\sum_{j=1}^{q} \beta_j \nu_j \right) = \operatorname{supp}\left(\sum_{i=1}^{p} \alpha_i \mu_i \right)
\end{equation*}
Now, on the one hand, $\operatorname{supp}\left(\sum_{i=1}^{p} \alpha_i \mu_i \right) \subseteq \bigcup_{\mu \in C_1} \operatorname{supp}(\mu)$. On the other hand, we have:
\begin{equation*}
    \operatorname{supp}\left(\sum_{j=1}^{q} \beta_j \nu_j \right) = \bigcup_{j:\beta_j \neq 0} \operatorname{supp}(\nu_j).
\end{equation*}
Since not all $\beta_j$'s can be $0$, we deduce that there exist $j$ such that $\operatorname{supp}(\nu_j) \subseteq \bigcup_{\mu \in C_1} \operatorname{supp}(\mu)$ which contradicts the assumption. Therefore, $\text{conv}(C_1)$ and $\text{conv}(C_2)$ are two disjoint convex sets in $\mathcal{V}$. They are also weak-* compact as the convex hulls of finite sets; indeed $\text{conv}(C_1)$ is the image of the compact probability simplex by the continuous map $(\alpha_1,\ldots,\alpha_p) \mapsto \sum_{i=1}^p \alpha_i \mu_i$, and similarly for $\text{conv}(C_2)$. Then, by Theorem \ref{thm:separating_convex}, we deduce that $\text{conv}(C_1)$ and $\text{conv}(C_2)$ and thus $C_1$ and $C_2$ can be separated by a test function $h \in C_0(\R^n\times S^{n-1})$.
\end{proof}
\begin{remark}
    Note that this only provides a sufficient but not necessary condition of separability. For instance, if $C_1$ and $C_2$ are both made of varifolds all having the same compact support set $K\subseteq \R^n \times S^{n-1}$ but such that their total mass satisfy $\mu(K)<\nu(K)$ for all $\mu \in C_1$ and $\nu \in C_2$, then $C_1$ and $C_2$ are still separable by taking any test function $h$ that is constant equal to $1$ on $K$.
\end{remark}
The above condition on the support of the varifolds becomes even more explicit when considering varifolds that are associated to continuous shapes via the mapping $q \in \mathcal{I} \mapsto \mu_q$ introduced in Section \ref{ssec:varifolds}. We get specifically:
\begin{corollary}\label{cor:separating_shapes}
    Let $\Sigma_1, \Sigma_2$ be finite sets of continuous shapes in $\mathcal{S}$ such that either 1) for each $q' \in \Sigma_2$, $q'(M)\not\subseteq \bigcup_{q\in \Sigma_1} q(M)$ or 2) for each $q \in \Sigma_1$, $q(M)\not\subseteq \bigcup_{q'\in \Sigma_2} q'(M)$. Then the two sets $\Sigma_1, \Sigma_2$ are strictly separable by a test function $h \in C_0(\R^n\times S^{n-1},\R)$.
\end{corollary}
This indeed follows directly from Theorem \ref{thm:separating_support} since for $q \in \mathcal{I}$, $\operatorname{supp}(\mu_q) = \{(q(m),v_q(m)): \ m \in M\}$ and thus the assumption of Corollary \ref{cor:separating_shapes} implies the assumption of Theorem \ref{thm:separating_support}. 
\begin{remark}
When the condition of Corollary \ref{cor:separating_shapes} is not satisfied, it is possible to find sets that are not separable within the SVarM setting. Figure \ref{fig:counterex_sep} shows such an example in which the two sets (green and red), each made of two shapes, cannot be separated. This is because the two convex hulls of the varifolds associated to the shapes in those sets do intersect in $\mathcal{V}$; specifically, one can see that the varifold $\frac{1}{2}\mu_{q_1} + \frac{1}{2} \mu_{q_2}$, $q_1$ and $q_2$ being the two shapes of the first (green) set, is equal to $\frac{1}{2}\mu_{q'_1} + \frac{1}{2} \mu_{q'_2}$ where $q'_1$ and $q'_2$ are the two shapes in the second (red) set, and therefore belongs to the intersection of the two convex hulls. This example further shows that the weaker and somewhat more natural condition that the support of any varifold in one of the two sets is not included in the support of any varifold from the other set is not sufficient to guarantee separability by SVarM since the example of Figure \ref{fig:counterex_sep} is not separable and does satisfy this condition (but does not satisfy that of Corollary \ref{cor:separating_shapes}).  
\end{remark}
\begin{figure}
    \centering    
    \includegraphics[height=.45\linewidth]{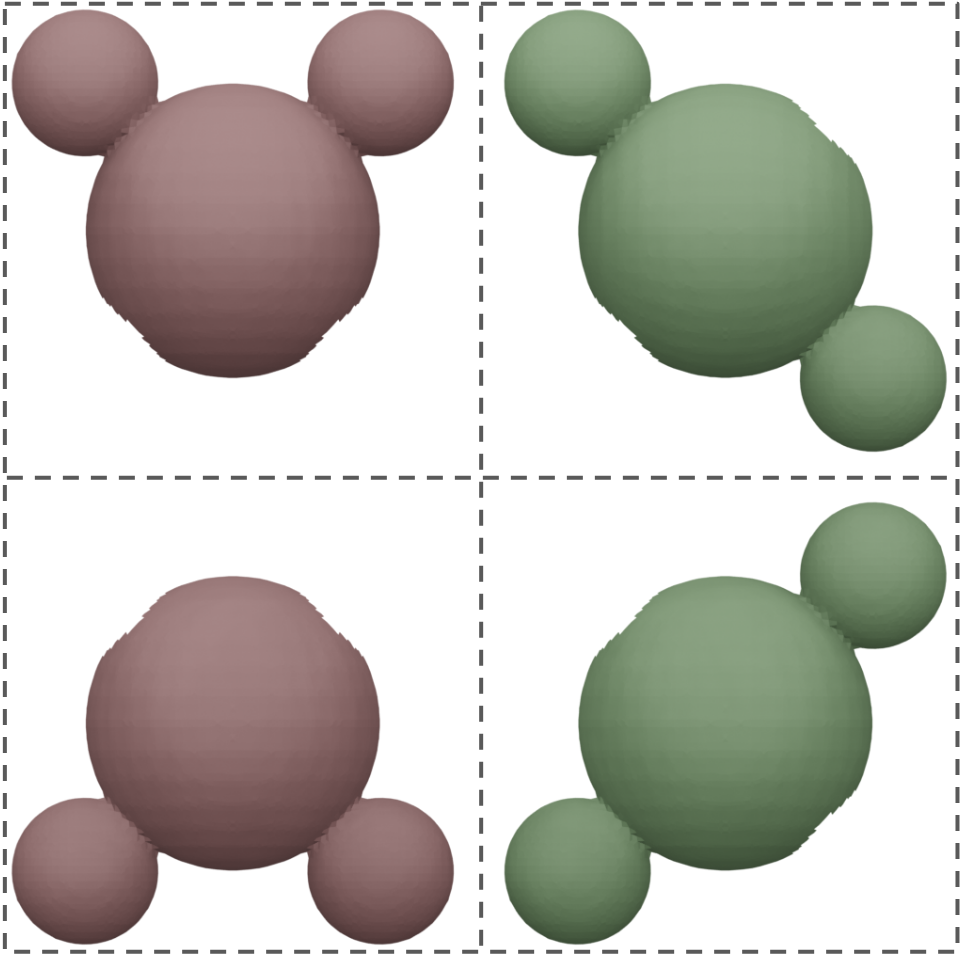}
    \caption{An example of two sets of shapes (in green and red) that cannot be separated within the SVarM framework. Note also that neither shape in one set is a subset of a shape from the other set, showing that this is not a sufficient condition of separability for SVarM.}\label{fig:counterex_sep}
\end{figure}
One can also obtain a corresponding result for discrete varifolds such as varifolds approximating discrete meshes. It is formulated below and is again a direct corollary of Theorem \ref{thm:separating_support}:
\begin{corollary}\label{cor:separating_discrete_var}
    Let $C_1$ and $C_2$ be two finite sets of discrete varifolds that are given by\\ $C_1=\left\{\sum_{i=1}^{N_k}r_i^{(k)} \delta_{(x_i^{(k)},v_i^{(k)})} \right\}_{k=1,\ldots,K_1}$ and $C_2=\left\{\sum_{i=1}^{\tilde{N}_k}\tilde{r}_i^{(k)} \delta_{(\tilde{x}_i^{(k)},\tilde{v}_i^{(k)})} \right\}_{k=1,\ldots,K_2}$. Assume that these sets satisfy either of the two following conditions:
    \begin{enumerate}
        \item For each $k=1,\ldots,K_1$, $\{(x_i^{(k)},v_i^{(k)}), i=1,\ldots N_k\} \not \subseteq \bigcup_{l=1}^{K_2}\{(\tilde{x}_i^{(l)},\tilde{v}_i^{(l)}), i=1,\ldots \tilde{N}_l\}$
        \item For each $k=1,\ldots,K_2$, $\{(\tilde{x}_i^{(k)},\tilde{v}_i^{(k)}), i=1,\ldots \tilde{N}_k\} \not \subseteq \bigcup_{l=1}^{K_1}\{(x_i^{(l)},v_i^{(l)}), i=1,\ldots N_l\}$
    \end{enumerate}
    Then the two sets $C_1, C_2$ are separable by a test function $h \in C_0(\R^n\times S^{n-1},\R)$.
\end{corollary}
%\begin{thm}
%    If $h,h' : \R^3 \times S^2 \to \R$ and $\mu\in\mathcal{M}(\R^3\times S^2)$  such that $\|h-h'\|_{L^1}<\delta$, then 
%    \begin{equation*}
%        \left|\int_{\R^3 \times S^2}h(x,v)\operatorname{d}\mu- \int_{\R^3 \times S^2}h'(x,v)\operatorname{d}\mu\right|<\delta \mu(\R^3\times S^2)    
%    \end{equation*}
%\end{thm}

%\subsubsection{Some Other Propositions}
%\begin{thm}
%    Let $K\subseteq\mathcal{M}(\R^3\times S^2)$ be a weak-compact* subset and $f:K\to\R$ a continuous function. Then for every $\epsilon>0$ there exists $\{h_i\}^n_{i=1}\subseteq C_0(\R^3\times S^2,\R)$ and $\{w_i\}^{n+1}_{i=1}, \{b_i\}^{n+1}_{i=1} \subseteq \R$ such that
%    \begin{equation}
%        \sup_{\mu\in K}\left| f(\mu) - w_{n+1}\sigma\left(w_i \int_{\R^3\times S^2} h_i d\mu + b_i \right) +b_{n+1} \right|<\epsilon
%    \end{equation}
%\end{thm}
%All the above results are stated with respect to a function $h$ in the general test function space $C_0(\R^n\times S^{n-1})$. As described in Section \ref{sec:netw_architecture} below, we will in practice represent the learnable function $h$ via a multilevel perceptron (MLP) neural network architecture. It is then straightforward, based on the many existing universal approximation and expressivity theorems \cite{hornik1989multilayer,cybenko1989approximation,haykin1994neural,lu2017expressive}, to rephrase those results with functions $h$ that are expressible by a sufficiently wide and/or deep MLP.
%\textcolor{red}{maybe we could move this paragraph somewhere in Section 3}

\subsubsection{Stability with Respect to Shape Deformations and Perturbations}
Another key theoretical question regarding the SVarM approach is the stability of the model under perturbations of the input data. We show below that this type of result can hold under the assumption that $h$ is Lipchitz continuous. As described in Section \ref{sec:netw_architecture}, we will approximate $h$ using a multilayer perceptron (MLP) architecture. For fixed weights (i.e., after training), an MLP is a Lipchitz function, and one can efficiently compute an upper bound on its Lipchitz constant \citep{NEURIPS2018_d54e99a6}. Let us first examine stability for general varifolds in $\mathcal{V}$:
\begin{thm}\label{thm:stabilityW1}
    Let $\mu,\nu\in\mathcal{V}$ and $h\in C_0(\R^n\times S^{n-1})$ such that $h$ is Lipchitz with constant $K$. Then, 
    \begin{equation*}
       |\langle\mu,h\rangle - \langle\nu,h\rangle|\leq\|h\|_\infty\cdot \|\mu-\nu\|_{\text{TV}}
    \end{equation*}
    and
    \begin{equation*}
        |\langle\mu,h\rangle - \langle\nu,h\rangle|\leq \left|\|\mu\|_{\text{TV}}-\|\nu\|_{\text{TV}}\right|\cdot\|h\|_\infty + \|\mu\|_{\text{TV}}\cdot K\cdot W_1\left(\frac{\mu}{m_\mu},\frac{\nu}{m_\nu}\right)
    \end{equation*}
    where $W_1$ denotes the 1-Wasserstein distance. 
\end{thm}
\begin{proof}
Let $\mu,\nu\in\mathcal{V}$ and $h\in C_0(\R^n\times S^{n-1})$ such that $h$ is Lipchitz with constant $K$. 

Bound 1: Based on the definition of the TV norm on $\mathcal{M}(\R^n \times S^{n-1})$ given above, we can bound the difference in the evaluation of $\mu$ and $\nu$ on $h\neq 0$ as follows:
\begin{align*}
    |\langle\mu,h\rangle-\langle\nu,h\rangle|=&\left|\int_{\R^n\times S^{n-1}} h\operatorname{d}\mu-\int_{\R^n\times S^{n-1}} h\operatorname{d}\nu\right|\\
    &\left|\int_{\R^n\times S^{n-1}} h\operatorname{d}(\mu-\nu)\right|\\
    &\left|\int_{\R^n\times S^{n-1}} \frac{h}{\|h\|_\infty}\operatorname{d}(\mu-\nu)\right| \|h\|_{\infty} \\
    &\leq \|\mu-\nu\|_{\text{TV}} \|h\|_{\infty}. 
    % =&\left|\int_{\R^n\times S^{n-1}} h \left(\frac{\operatorname{d}\mu}{\operatorname{d}\gamma}\right)\operatorname{d}\gamma-\int_{\R^n\times S^{n-1}} h\left(\frac{\operatorname{d}\nu}{\operatorname{d}\gamma}\right)\operatorname{d}\gamma\right|\\
    % =&\left|\int_{\R^n\times S^{n-1}} h \left(\frac{\operatorname{d}\mu}{\operatorname{d}\gamma}-\frac{\operatorname{d}\nu}{\operatorname{d}\gamma}\right)\operatorname{d}\gamma\right|\\
    % \leq&\|h\|_{\infty}\int_{\R^n\times S^{n-1}} \left|\frac{\operatorname{d}\mu}{\operatorname{d}\gamma}-\frac{\operatorname{d}\nu}{\operatorname{d}\gamma}\right|\operatorname{d}\gamma=\|h\|\infty\cdot\|\mu-\nu\|_{TV}.
\end{align*}

Bound 2: For brevity, let us define the total masses $m_\mu=\|\mu\|_{\text{TV}}, m_\nu=\|\nu\|_{\text{TV}}$ and the corresponding normalized measures $\overline{\mu} = \frac{\mu}{m_\mu},\overline{\nu} = \frac{\nu}{m_\nu}$. We begin by with a series of inequalities:
\begin{align*}
    |\langle\mu,h\rangle-\langle\nu,h\rangle|=& \left|m_\mu\left\langle\overline{\mu},h\right\rangle-m_\nu\left\langle\overline{\nu},h\right\rangle\right|\\ =& \left|m_\mu\left\langle\overline{\mu},h\right\rangle-m_\mu\left\langle\overline{\nu},h\right\rangle+m_\mu\left\langle\overline{\nu},h\right\rangle-m_\nu\left\langle\overline{\nu},h\right\rangle\right|\\ 
    \leq&\left|m_\mu\left\langle\overline{\nu},h\right\rangle-m_\nu\left\langle\overline{\nu},h\right\rangle\right|+\left|m_\mu\left\langle\overline{\mu},h\right\rangle-m_\mu\left\langle\overline{\nu},h\right\rangle\right|\\
    \leq& |m_\mu-m_\nu|\cdot\left|\int_{\R^n\times S^{n-1}}h \operatorname{d} \overline{\nu}\right|+m_\mu\left|\int_{\R^n\times S^{n-1}} h \operatorname{d} \overline{\mu} - \int_{\R^n\times S^{n-1}} h \operatorname{d} \overline{\nu}\right|\\
    \leq& |m_\mu-m_\nu|\cdot\|h\|_\infty+m_\mu\left|\int_{\R^n\times S^{n-1}} h \operatorname{d} \overline{\mu} - \int_{\R^n\times S^{n-1}} h \operatorname{d} \overline{\nu}\right|. 
\end{align*}
Next we define the function $\overline{h} := \frac{h}{K}$. Since the Lipchitz constant of $h$ is $K$, it follows that $\overline{h}$ is $1-$Lipchitz. Thus we can write
\begin{align*}
    \left|\int_{\R^n\times S^{n-1}} h \operatorname{d} \overline{\mu} - \int_{\R^n\times S^{n-1}} h \operatorname{d} \overline{\nu}\right| = &K\left|\int_{\R^n\times S^{n-1}} \overline{h} \operatorname{d} \overline{\mu} - \int_{\R^n\times S^{n-1}} \overline{h} \operatorname{d} \overline{\nu}\right|.
\end{align*}
By the Kantorovich-Rubinstein formulation of the 1-Wasserstein distance,
\begin{equation*}
    W_1(\overline{\mu},\overline{\nu}) = \sup_{\|f\|_{W^{1,\infty}}\leq 1}\left|\int_{\R^n\times S^{n-1}} f \operatorname{d} \overline{\mu} - \int_{\R^n\times S^{n-1}} f \operatorname{d} \overline{\nu}\right|\geq\left|\int_{\R^n\times S^{n-1}} \overline{h} \operatorname{d} \overline{\mu} - \int_{\R^n\times S^{n-1}} \overline{h} \operatorname{d} \overline{\nu}\right|
\end{equation*}
and we conclude that:
\begin{equation*}
    |\langle\mu,h\rangle-\langle\nu,h\rangle| \leq |m_\mu-m_\nu|\cdot\|h\|_\infty+m_\mu\cdot K\cdot W_1(\overline{\mu},\overline{\nu}).
\end{equation*}
\end{proof}
The first bound involving the total variation between $\mu$ and $\nu$ implies stability of $\langle \mu, h\rangle$ under partial occlusion of some parts of a shape (c.f. experiments of Section \ref{ssec:robustness}) while the second bound can be thought of as a control of the changes in $\langle \mu, h\rangle$ under variations of the total mass and support. Our second result demonstrates that the outputs of SVarM are stable in the space of shapes with respect to sufficiently regular perturbations in $\mathcal{S}$:
\begin{thm}\label{thm:stability_shape}
    Let $h\in C_0(\R^n\times S^{n-1})$ such that $h$ is Lipschitz with constant $K$ and $q,q'$ be bi-Lipschitz embeddings from a $d$-dimensional manifold $M$ into $\R^{n}$ (with either $d=1$ or $d=n-1$) such that $\|q - q'\|_{W^{1,\infty}} \leq C$. Then 
    \begin{equation}
        |\langle\mu_{q},h\rangle - \langle\mu_{q'},h\rangle| \leq CA_q(K\sqrt{2dL^{2d-1} + 1}+d^{d+1/2}L^{4d-1}\|h\|_{\infty})
    \end{equation}
    where $A_q = \int_M\operatorname{d vol}_q$ is the total length/area of $q(M)$, and $L$ the maximum of the Lipschitz constants of $q$ and $q'$.
\end{thm}
\begin{proof}
    Let $h\in C_0(\R^n\times S^{n-1})$ be such that $h$ is Lipschitz with constant $K$ and $q,q'$ be bi-Lipschitz embeddings from a $d$-dimensional manifold $M$ into $\R^{n}$ such that $\|q - q'\|_{W^{1,\infty}} \leq C$. We estimate the difference between the actions of the varifolds associated with $q$ and $q'$ on $h$ by:
\begin{align}    
    &|\langle\mu_q,h\rangle-\langle\mu_{q'},h\rangle|= \left|\int_{M} h(q,v_{q}) \operatorname{d} \operatorname{vol}_q -  \int_{M}h(q',v_{q'}) \operatorname{d}\operatorname{vol}_{q'}\right|\\
    =&\left|\int_{M} h(q,v_{q})\sqrt{|Jq|} - h(q',v_{q'})\sqrt{|Jq|}+ h(q',v_{q'})\sqrt{|Jq|} - h(q',v_{q'})\sqrt{|Jq'|} \operatorname{d}m\right|\\
    \leq&\int_{M} \left|h(q,v_q)-h(q',v_{q'})\right| \sqrt{|Jq|} \operatorname{d}m + \int_{M} |h(q',v_{q'})|\frac{\left|\sqrt{|Jq|}-\sqrt{|Jq'|}\right|}{\sqrt{|Jq|}} \sqrt{|Jq|}\operatorname{d}m.
\end{align}
where $Jq(m)=Dq(m)^T Dq(m)$ and $|Jq(m)|$ is the absolute value of the determinant of $Jq(m)$. Since $q,q'$ are bi-Lipchitz embeddings, there exists $L>0$ such that the singular values of $Dq$ and $Dq'$ are in $[1/L,L]$ almost everywhere on $M$. Therefore, for almost every $m\in M$, $\frac{1}{L^d}<\sqrt{|Jq|},\sqrt{|Jq'|}<L^d$ and $\frac{\sqrt{d}}{L}<\|Dq\|_F,\|Dq'\|_F<\sqrt{d}L$.
Moreover, for almost every $m\in M$, 
\begin{align*}
    \frac{\left|\sqrt{|Jq|}-\sqrt{|Jq'|}\right|}{\sqrt{|Jq|}} &\leq \frac{L^{2d}}{2}\left||Jq|-|Jq'|\right|.
\end{align*}
To control the differences in the Jacobian determinants we leverage the following perturbation bound for determinants using the Frobenius norm.
\begin{lemma}[Theorem 2.12 of \citep{ipsen2008perturbation}] Let $A$ and $E$ be $d\times d$ matrices. Then
    $||A|-|A+E||\leq d\|E\|_F \max(\|A\|_F,\|A+E\|_F)^{d-1}.$
\end{lemma}
Substituting $A=Jq$ and $E = Jq-Jq'$ we have directly that
\begin{equation*}
    ||Jq|-|Jq'||\leq d\|Jq-Jq'\|_F \max(\|Jq\|_F,\|Jq'\|_F)^{d-1}
\end{equation*}
Since $\|Jq\|_F\leq\|Dq\|_F^2 \leq dL^2$ and $ \|Jq'\|_F\leq\|Dq'\|_F^2\leq dL^2$, it follows that \begin{equation*}
    ||Jq|-|Jq'||\leq d\|Jq-Jq'\|_F (dL^2)^{d-1} = d^d\|Jq-Jq'\|_FL^{2(d-1)}
\end{equation*}
Finally, to control the Frobenius norm of the difference of the Jacobians of $q$ and $q'$ we use,
\begin{align*}
    \|Jq-Jq'\|_F\leq (\|Dq\|_F+\|Dq'\|_F)\|Dq-Dq'\|_F\leq 2\sqrt{d}LC.
\end{align*}
So, combining these inequalities we obtain
\begin{equation*}
    \int_{M} |h(q,v_q)|\frac{\left|\sqrt{|Jq|}-\sqrt{|Jq'|}\right|}{\sqrt{|Jq|}} \sqrt{|Jq|}\operatorname{d}m\leq d^{d+1/2}CL^{4d-1}\int_M |h(q,v_q)|\sqrt{|Jq|}\operatorname{d}m \leq d^{d+1/2}CL^{4d-1}\|h\|_{\infty}A_q.
\end{equation*}
where $A_q=\int_M \sqrt{|Jq|}\operatorname{d}m = \int_M \operatorname{d vol}_q$. Moreover, $v_q$ is the unit tangent (or normal) vector of $q$ computed by $v_q = \frac{\partial_1q\wedge\ldots\wedge\partial_dq}{\sqrt{|J_q|}}$. Therefore, for almost every $m\in M$, 
\begin{align*}
    \|v_q-v_{q'}\|&\leq 2L^d\|(\partial_1q\wedge...\partial_dq)-(\partial_1q'\wedge...\partial_dq')\|\leq2L^d\sum_{i=1}^d\|(\partial_1q\wedge...\wedge(\partial_iq-\partial_iq')\wedge...\wedge\partial_dq')\|\\
    &\leq2L^d\sum_{i=1}^d \|\partial_1q\|\cdot...\cdot\|\partial_iq-\partial_iq'\|\cdot...\cdot\|\partial_dq'\|\leq 2dL^{2d-1} \|Dq-Dq'\|\leq 2dL^dC.\\
\end{align*}
where we have again used Hadamard's inequality to bound the norm of the exterior product. So, $\sqrt{\|q-q'\|^2+\|v_q-v_{q'}\|^2}\leq  C\sqrt{2dL^{2d-1} + 1}$ almost everywhere.
Since $h$ is Lipschitz with constant $K$, we have for almost every $m \in M$,
\begin{equation*}
    |h(q, v_q) - h(q', v_{q'})| \leq K \sqrt{\|q-q'\|^2+\|v_q-v_{q'}\|^2} \leq CK\sqrt{2dL^{2d-1} + 1} .
\end{equation*} 
Therefore,
\begin{align*}
    \int_{M} \left|h(q,v_q)-h(q',v_{q'})\right| \sqrt{|Jq|} \operatorname{d}m \leq CK\sqrt{2dL^{2d-1} + 1}\int_M\sqrt{|Jq|}\operatorname{d}m = CKA_q\sqrt{2dL^{2d-1} + 1}
\end{align*}
Thus, we have obtained that:
\begin{equation*}
    |\langle\mu_q,h\rangle-\langle\mu_{q'},h\rangle|\leq CA_q(K\sqrt{2dL^{2d-1} + 1}+d^{d+1/2}L^{4d-1}\|h\|_{\infty}).
\end{equation*}
\end{proof}

\section{SVarM regression and classification models}
\label{sec:SVarM_training}
The SVarM framework for regression or classification consists in estimating a map $h:\R^n \times S^{n-1}\to \R$ (for $n=2$ or $3$) which, together with the bias $\beta$, represents any weak-* continuous affine map on $\mathcal{V}$ as we have previously shown in Proposition \ref{prop:SVarM_affine_maps}. In the absence of any clear finite-dimensional family of functions to represent $h$, we focus here on non-parametric models; we first examine the situation in which $h$ is taken in some reproducing kernel Hilbert space of functions on $\R^n \times S^{n-1}$ leading to a formulation akin to kernel regression/classification, before relying instead on the representation of $h$ via a neural network model which we argue is a more effective approach when dealing with large data sets. From now on, we shall assume that a training set $\{\mu^{(k)},y^{(k)}\}_{k=1,\ldots,R}$ is given, where $\{\mu^{(1)},..., \mu^{(R)}\}$ is a set of varifolds (computed from the corresponding input shapes) and $\{y^{(1)}, ..., y^{(R)}\}$ are the corresponding labels (either real numbers in regression problems or integers giving class membership for classification).

\subsection{Reproducing kernel Hilbert space approach}
Reproducing kernel Hilbert spaces (RKHS) in the context of varifolds have been considered in several previous works \citep{charon2013varifold,kaltenmark2017general,hsieh2021metrics,paul2024sparse}, primarily as a way to build explicit metrics on the space of curves and surfaces. One can rely on a similar construction in order to specify the space of test functions $h$, in which case the problem is related to \emph{kernel ridge regression} with a theoretically explicit form for the solution. Before discussing this approach, we begin with a brief recap of the main background material on RKHS, referring the reader to e.g. \citep{Aronszajn1950} or \citep{younes2019shapes} Chap. 8 for more detailed presentations.

Let us start by introducing a positive definite kernel function $K$ on $\R^n \times S^{n-1}$, i.e. a function $K:(\R^n \times S^{n-1}) \times (\R^n \times S^{n-1}) \rightarrow \R$ which is symmetric and satisfies the property that for any $(x_i,v_i)_{i=1,\ldots,m} \in (\R^n \times S^{n-1})^m$ and $(w_i) \in \R^m$, it holds that:
\begin{equation*}
 \sum_{i,j=1}^{m} w_i K((x_i,v_i),(x_j,v_j)) w_j \geq 0
\end{equation*}
with strict inequality as long as all points are distinct and $(w_i)_{i=1,\ldots,m} \neq 0$, in other words the matrix $[K((x_i,v_i),(x_j,v_j))]_{i,j}$ is a symmetric positive definite $m\times m$ matrix. Such kernel functions can in particular be constructed, as in \citep{charon2013varifold,kaltenmark2017general}, by taking the tensor product of a positive definite kernel on $\R^n$ and $S^{n-1}$, for which there are many known kernel families. If one makes the additional assumption that $K$ is a continuous function vanishing at infinity, it follows from the classical theory of kernels, c.f. \citep{younes2019shapes}, that there exists a unique RKHS $W$ of functions on $\R^n \times S^{n-1}$ such that $W$ is continuously embedded into $C_0(\R^n \times S^{n-1})$ which is ``associated'' to $K$. The previous statement means specifically that the Hilbert inner product on $W$ between any two functions $f(\cdot) = K((x,v), \cdot)$ and $g(\cdot) = K((x',v'), \cdot)$ is explicitly given by $\langle f,g \rangle_W = K((x,v),(x',v'))$. More general continuous embeddings of $W$ into spaces $C_0^s(\R^n \times S^{n-1})$ for $s \in \mathbb{N}$ can be also recovered under additional assumptions on the regularity of the kernel $K$.

Let us now consider the regression problem (the case of classification can be treated in a similar fashion). It can be framed, in this setting, similarly to a kernel ridge regression problem:
\begin{equation}
\label{eq:regression_kernel}
 \min_{h \in W, \beta \in \R} \, \lambda \|h\|_{W}^2 + \frac{1}{R} \sum_{k=1}^R \left( \langle \mu^{(k)},h\rangle + \beta - y^{(k)} \right)^2
\end{equation}
where $\lambda >0$ weighs the regularization term given by the RKHS norm $\|\cdot\|_{W}$. By direct differentiation with respect to $\beta$, one finds the linear equation satisfied by the optimal $(h_*,\beta_*)$:
\begin{equation}
\label{eq:regression_kernel1}
 \beta_* + \frac{1}{R}\sum_{k=1}^{R} \langle \mu^{(k)},h_*\rangle = \bar{y}.
\end{equation}
where $\bar{y} = \sum_{k=1}^{R} y^{(k)} /R$ is the average value of the labels. We now derive an explicit optimality condition on $h$ by focusing on the practical situation in which each $\mu^{(k)}$ is a discrete varifold that we write $\mu^{(k)} = \sum_{i=1}^{m_k} w_{i}^{(k)}\delta_{(x_{i}^{(k)},v_{i}^{(k)})}$. In that case, the second term of \eqref{eq:regression_kernel} only depends on the values of $h$ at the points $(x_i^{(k)},v_i^{(k)})$. In order to remove repetitions, we shall denote by $(\tilde{x}_i,\tilde{v}_i)$ for $i=1,\ldots,\tilde{m}$ (with $\tilde{m} \leq m \doteq m_1+\ldots+m_R$) the subset of distinct points in $\R^n \times S^{n-1}$ among the $(x_i^{(k)},v_i^{(k)})$'s. This implies that the optimal $h^*$ must lie in the orthogonal subspace $W_0^\bot$ where:
\begin{equation*}
 W_0 = \{h \in W: h(\tilde{x}_i,\tilde{v}_i) = 0 \text{ for all } i=1,\ldots,\tilde{m}\}
\end{equation*}
Indeed, if $h^* \notin W_0^\bot$, then, denoting $h^*_\bot$ its orthogonal projection onto $W_0^\bot$, we would have that $\|h^*_\bot\|_{W}^2 < \|h^*\|_W^2$ while the second term in \eqref{eq:regression_kernel} remains the same for $h^*_\bot$ due to the definition of $W_0$. Thus $h^*$ would not be optimal. From standard properties of RKHS, specifically Lemma 8.1 in \citep{younes2019shapes}, it follows that $h_*$ must take the form:
\begin{equation*}
 h_*(\cdot) = \sum_{i=1}^{\tilde{m}} \alpha_i K((\tilde{x}_i,\tilde{v}_i),\cdot)
\end{equation*}
where $\alpha = (\alpha_i)\in\R^{\tilde{m}}$ is to be determined. Introducing $\tilde{\mathbf{K}} = [K((\tilde{x}_i,\tilde{v}_i),(\tilde{x}_j,\tilde{v}_j))]$ the matrix of the kernel evaluations between all pairs of points, which is positive definite, one has $\|h_*\|_W^2 = \alpha^T \tilde{\mathbf{K}} \alpha$.
On the other hand, we may define a second kernel matrix $\mathbf{K} = [K((x_i^{(k)},v_i^{(k)}),(\tilde{x}_j,\tilde{v}_j))] \in \R^{m \times \tilde{m}}$ so that the energy to minimize in \eqref{eq:regression_kernel} becomes:
\begin{equation*}
    J(\alpha) =  \lambda \alpha^T \tilde{\mathbf{K}} \alpha + \frac{1}{R}\sum_{k=1}^{R} \left((w^{(k)})^T (\Pi^{(k)})^T \mathbf{K} \alpha + \beta - y^{(k)}\right)^2
\end{equation*}
in which $w^{(k)} \in \R^{m_k}$ denotes the $k$-th block vector of $w$ and $\Pi^{(k)} \in \R^{m_k \times m}$ is the orthogonal projection that extracts the $k$-th block entries of a vector in $\R^{m}$. This amounts to the minimization of a quadratic function of $\alpha$ for which the optimality condition is:
\begin{equation*}
    \left[\lambda\tilde{\mathbf{K}} + \mathbf{K}^T\left(\frac{1}{R} \sum_{k=1}^R (\Pi^{(k)})^T w^{(k)} (w^{(k)})^T  \Pi^{(k)} \right) \mathbf{K} \right]\alpha_* + \mathbf{K}^T(w/R) \beta= \mathbf{K}^T(w \otimes y/R) 
\end{equation*}
where we use the notation $w \otimes y$ for the vector made of the $R$ blocks $y^{(k)} w^{(k)}$, $k=1,\ldots,R$. Ultimately, combining the above with \eqref{eq:regression_kernel1}, we see that the optimal $\alpha_*,\beta_*$ (and thus the solution of the kernel regression problem) are obtained by solving the linear system:
\begin{equation}
\label{eq:regression_kernel_solution}
    \begin{bmatrix}
        \lambda\tilde{\mathbf{K}} + \mathbf{K}^T\left(\frac{1}{R} \sum_{k=1}^R (\Pi^{(k)})^T w^{(k)} (w^{(k)})^T  \Pi^{(k)} \right) \mathbf{K}  & \mathbf{K}^T(w/R) \\
        w^T \mathbf{K}/R & 1
    \end{bmatrix}
    \begin{bmatrix}
        \alpha_* \\
        \beta_*
    \end{bmatrix}= 
    \begin{bmatrix}
        \mathbf{K}^T(w \otimes y/R) \\
        \bar{y}
    \end{bmatrix}
\end{equation}
We note that the above block matrix is invertible for any $\lambda >0$, as can be checked using the Schur complement criterion for block matrices. Despite the nice and simple mathematical structure of the problem, it should be pointed that the resulting linear system \eqref{eq:regression_kernel_solution} has a size essentially of the order of the total number of Diracs in the whole training set, which can be over 100 millions for some of the experiments of Section \ref{sec:numerical}. Furthermore, this linear system may be very ill-conditioned if $\lambda$ is not properly calibrated. Thus, this direct strategy is likely not practical except for small datasets of shapes with limited number of vertices. Although this is out of the scope of the present paper, we suspect that such a kernel-based approach can stil be made practical if combined with some sparse approximation schemes allowing to reduce the dimensionality of the solution space, such as the quantization approach of \citep{hsieh2021metrics} or the Nystrom method considered in \citep{paul2024sparse}. In addition, one could leverage recently proposed techniques for this class of linear systems, such as methods that rely on preconditioning \citep{avron2017faster,zhao2024adaptive} or multilevel Newton schemes \citep{zhang2023fast}. Yet, in order to avoid any such preprocessing steps, we will consider, in the next section, another representation for the test function space. 

\subsection{Neural network model}
\label{sec:netw_architecture}
As an alternative to the above kernel strategy, we shall introduce a more direct approach that consists in modeling the test function via a neural network, which is now ubiquitous in the modern machine learning literature. We thus approximate $h$ using a multilevel perceptron (MLP), namely consider a set of trainable functions $h_\theta:\R^{2n}\to\R$ where $\theta$ denotes the trainable parameters of the MLP. Note that, based on the many existing universal approximation and expressivity theorems for MLPs \citep{hornik1989multilayer,cybenko1989approximation,haykin1994neural,lu2017expressive}, one can easily adapt the separability properties from Section \ref{sec:theo_res} for functions $h$ expressible by a sufficiently wide and/or deep MLP. Informally, if two sets of varifolds (resp. shapes) satisfy the assumptions of Theorem \ref{thm:separating_convex} (resp. Theorem \ref{thm:separating_support}), then one can find a separating test function for those two sets that can be expressed with a sufficiently wide/deep MLP. Given a discrete varifold $\mu = \sum_{i=1}^m w_i \delta_{(x_i,v_i)}$, we pass each support $(x_i,v_i)\in\R^n\times S^{n-1}\subseteq\R^{2n}$ through our MLP to compute a vector $H = [h_\theta(x_1,v_1),...,h_\theta(x_m,v_m)]$. We then view the weights of $\mu$ as a vector $W = [w_1,...,w_m]$ and compute $$H^TW+\beta=\sum_{i=1}^mw_ih_\theta(x_i,v_i) +\beta=\langle \mu,h_\theta\rangle+\beta$$ 
where $\beta$ is a trainable bias term. As an important note, the number of support Diracs $m$ has no influence on the network architecture and the number of trainable parameters in $(\theta,\beta)$. Those are determined by the dimension $n$ of the ambient space. The number $m$ only affects the evaluation of the network on the varifold $\mu$. For the linear regression problem in $\mathcal{V}$, we train this model in a supervised manner. Similar to the previous kernel approach, we minimize the mean squared error loss to the known labels:
\[
\operatorname{MSE}(\theta,\beta) = \frac{1}{R} \sum_{j=1}^R \left( \langle \mu_j,h_\theta\rangle + \beta - y_j \right)^2.
\]
\begin{figure}[h!]
\centering
  \includegraphics[width=.9\textwidth]{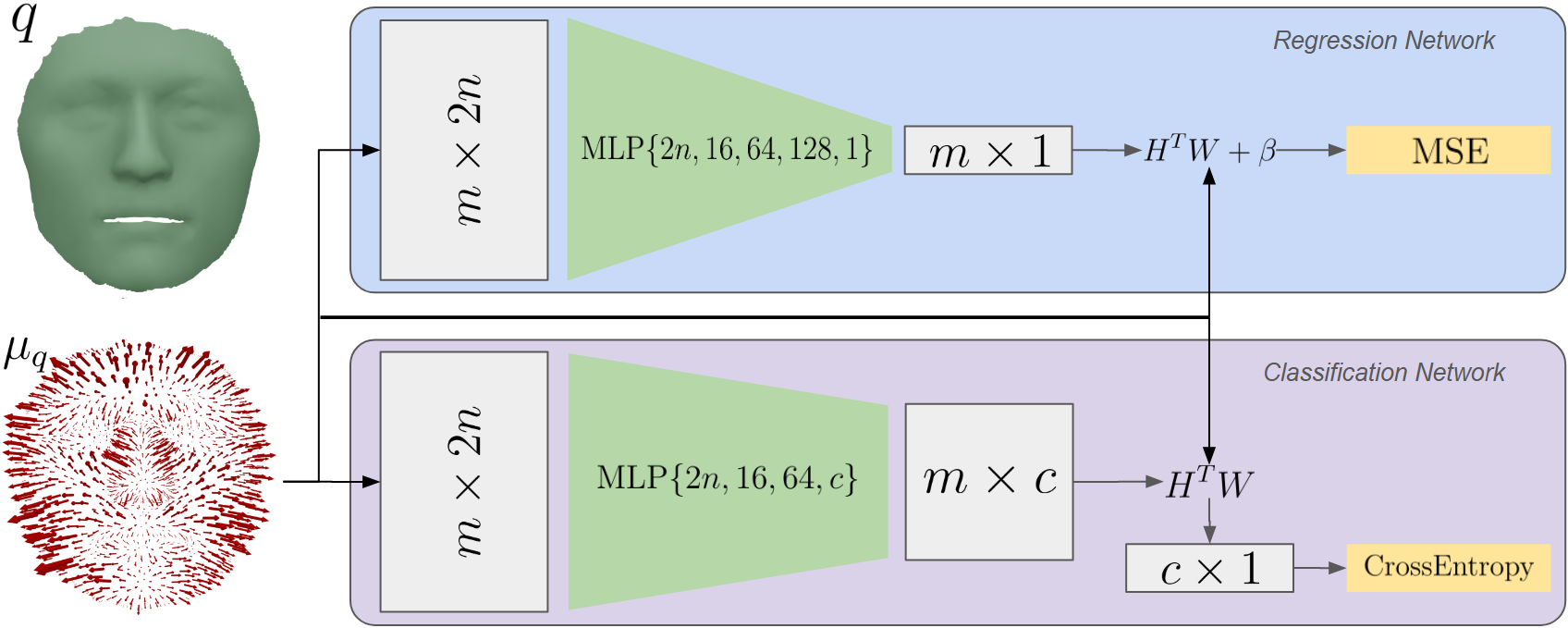} %\hfill\vline\hfill\includegraphics[width=.33\textwidth]{SVarM.png} 
  \caption{Model Summary: A mesh $q$ is mapped to its varifold representation $\mu_q$. The supports of the varifold are passed through an MLP and the resulting values are integrated with respect to the masses of $\mu_q$. For regression, a trainable bias term is added to integral. For classification, the integral produces a $c$-dimensional vector viewed as logits of the probability that $q$ belongs to each class.}
  \label{fig:model_schematic}
\end{figure}
To perform multi-class classification of varifolds, we slightly modify the architecture to learn a vector valued functions $h:\R^n \times S^{n-1}\to \R^c$ where $\int_{\R^n\times S^{n-1}} h \operatorname{d}\mu$ is a $c$-dimensional vector where each entry is interpreted as a logit of the probability that $\mu$ belongs to the corresponding class. The function $h$ is again modeled via a multi-layer perceptron (MLP) $h_\theta:\R^{2n}\to\R^c$. For $\mu = \sum_{i=1}^m w_i \delta_{(x_i,u_i)}$, we evaluate $h_\theta$ on the supports to obtain an $m\times c$ matrix $H$ with each row being the vector $h_\theta(x_i,u_i)$. Therefore, 
$H^TW = \langle \mu, h_{\theta}\rangle$
yields a $c$-dimensional logit vector. These logits are passed through a softmax and used with the standard cross-entropy loss for classification. A schematic of both models is outlined in Figure \ref{fig:model_schematic}.

\begin{remark}[Invariance to Orientation]
    It is important to note that the model presented so far relies on the orientation of the tangent/normals to the shapes, meaning that one assumes a consistent orientation of the different shapes in the data set. For certain applications that involve non-oriented data, it is necessary to make the SVarM approach invariant to the orientation of $v_i$. This can easily be achieved by taking $\widehat{h_\theta} = h_\theta(x,v)+ h_\theta(x,-v)$ and computing $\langle\mu,\widehat{h_\theta}\rangle+\beta$. In particular, in the applications to shape graphs of Section \ref{sec:numshapegraphs}, we rely on this modification to remain robust to reorientation of the edges of the shape graphs.
\end{remark}

\section{Numerical Experiments}\label{sec:numerical}
In this section, we present a series of numerical experiments designed to evaluate the performance, generalizability, and efficiency of our proposed model across a range of tasks involving 3D data. We consider three distinct scenarios: rotational alignment of 3D human body scans, classification of remeshed surface representations of MNIST digits and classification of facial graph and surface data extracted from the COMA data set. These experiments demonstrate the flexibility of our model across diverse input types and provide quantitative and qualitative results with comparisons to baseline methods. For each task, we outline the data preprocessing steps, describe the experimental setup, and report results using standard performance metrics. In the final section, we present a series of ablation and robustness experiments to highlight the stability of our model under noise and occlusions. These analyses provide deeper insight into the factors driving the model’s performance and demonstrate its robustness across a range of challenging conditions. All of the models utilize sigmoid activation functions and are trained using the Adam optimizer with a fixed learning rate of 0.005 for 100 epochs to ensure consistency in the presented results. All experiments were conducted on a standard home PC equipped with an Intel 3.2 GHz CPU and a GeForce GTX 2070 GPU (1620 MHz). A Pytorch implementation of the SVarM framework is publicly available on GitHub at \url{https://github.com/SVarM25/SVarM}.

\subsection{Regression for Single Axis Rotation Alignment of Human Bodies}\label{ssec:regression_exp}
Our first experiment evaluates the effectiveness of our regression model on a toy data set consisting of 3D human body scans from the FAUST data set \citep{Bogo:CVPR:2014} rotated along a fixed axis. In many real-world applications involving 3D data—such as human pose estimation, animation, or medical imaging—scans are often misaligned due to variations in sensor orientation, subject posture, or scanning artifacts. Such rotational misalignment can significantly degrade the performance of downstream tasks. Therefore, estimating and correcting the rotation of 3D scans is a crucial preprocessing step.
\begin{figure}[h!]
  \centering
       \includegraphics[width=1.0\textwidth]{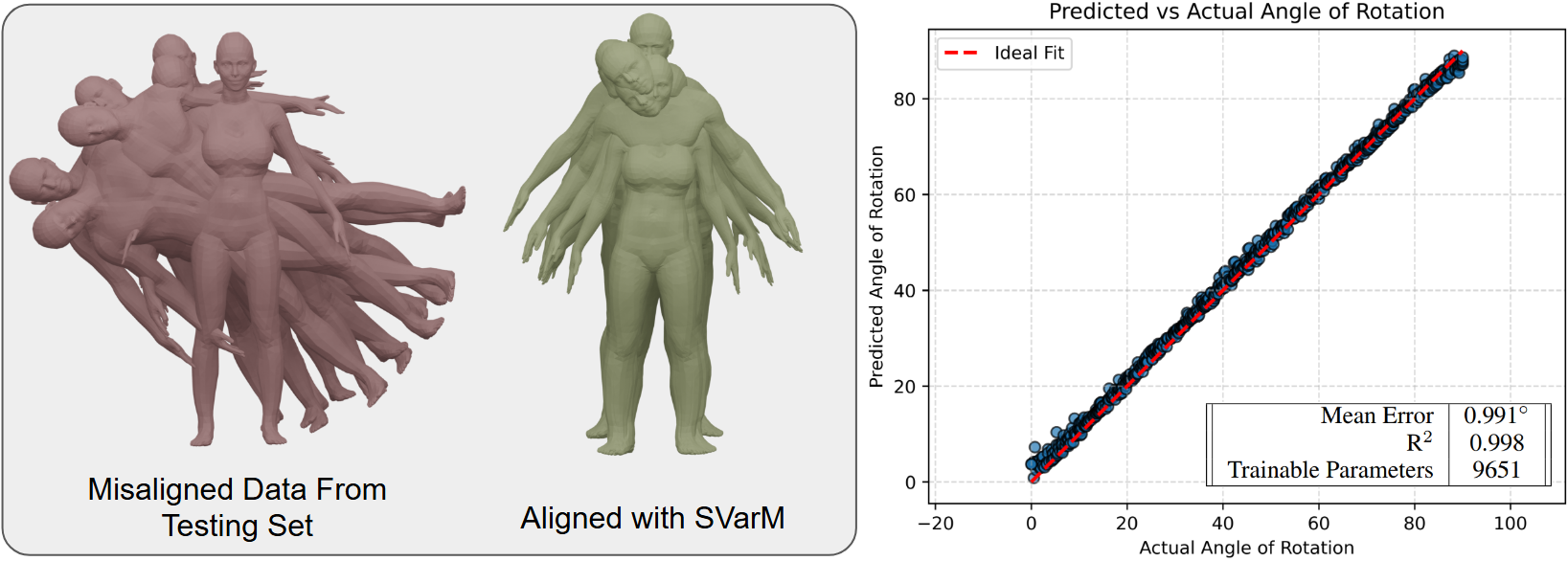}
       \includegraphics[width=1.0\textwidth]{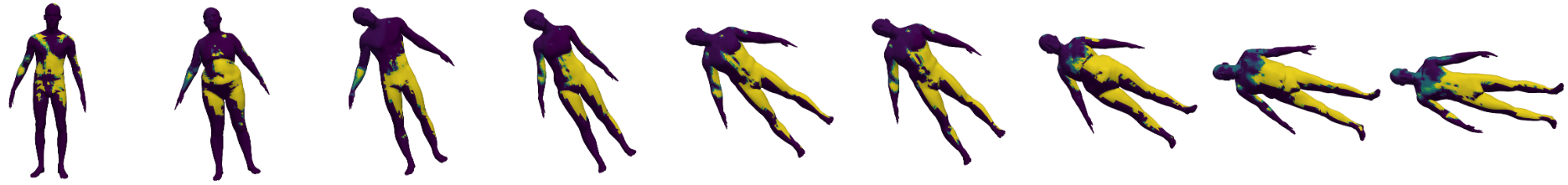}
      \caption{Regression on rotation angle around a fixed axis. We display a selection of human scans from our testing set of misaligned data (red) and present same data rotated by the angle predicted by the SVarM regression model (yellow). Additionally, we plot of the predicted angles against the ground truth on the entire testing set (right) reporting the mean error and correlation coefficient on the entire test set. Lastly, we display the learned function $h$ evaluated on several meshes from the testing data set; the color of each face corresponds to $h(c_f,v_f)$ ranging from blue (low values) to yellow (higher values).}
      \label{fig:rotation_regression} 
  %\hfill
  %\begin{minipage}[b]{0.45\textwidth}
  %  \centering
  %      \begin{tabular}{|| r | c ||} 
  %      \hline
  %      Mean Error & 0.991$^\circ$\\ 
  %      R$^2$& 0.998\\
  %      Trainable Parameters&9651\\
  %      \hline
  %      \end{tabular}
  %\end{minipage}
\end{figure}

To address this, we frame rotational alignment as a regression problem, where the model learns to predict the angle by which a given mesh has been rotated. The goal is to realign each mesh to a canonical upright position. We generate our data set by applying known rotations to a subset of the FAUST 3D human body scans along a single axis, storing the corresponding angles as labels resulting in ~8000 labeled meshes. Our model is trained on ~90\% of this data and evaluated on the remaining ~10\%. In Figure~\ref{fig:rotation_regression}, we present a selection of data set samples and demonstrate the model’s ability to accurately predict rotation angles and realign the meshes. Moreover, to highlight the interpretability of our model, we display the evaluation of the learned test function $h$ on samples from the data set where the color of each face of the mesh corresponds to the evaluation of the test function at $(c_f,n_f)$. We note that the meshes that are misaligned by a larger angle have a higher density of faces with larger values of $h$.
\subsection{Regression for Full Rotation Alignment of Human Data}\label{ssec:full_rotation_exp}
Building on our initial experiment with scalar angle regression, we next evaluate the ability of our model to predict full 3D rotation matrices. While predicting rotation around a fixed axis simplifies the problem, many real-world scenarios involve arbitrary rotations in 3D space. Consequently, modeling the full rotation is essential for general applicability to object alignment.

\begin{figure}[h!]
  \centering
       \includegraphics[width=\textwidth]{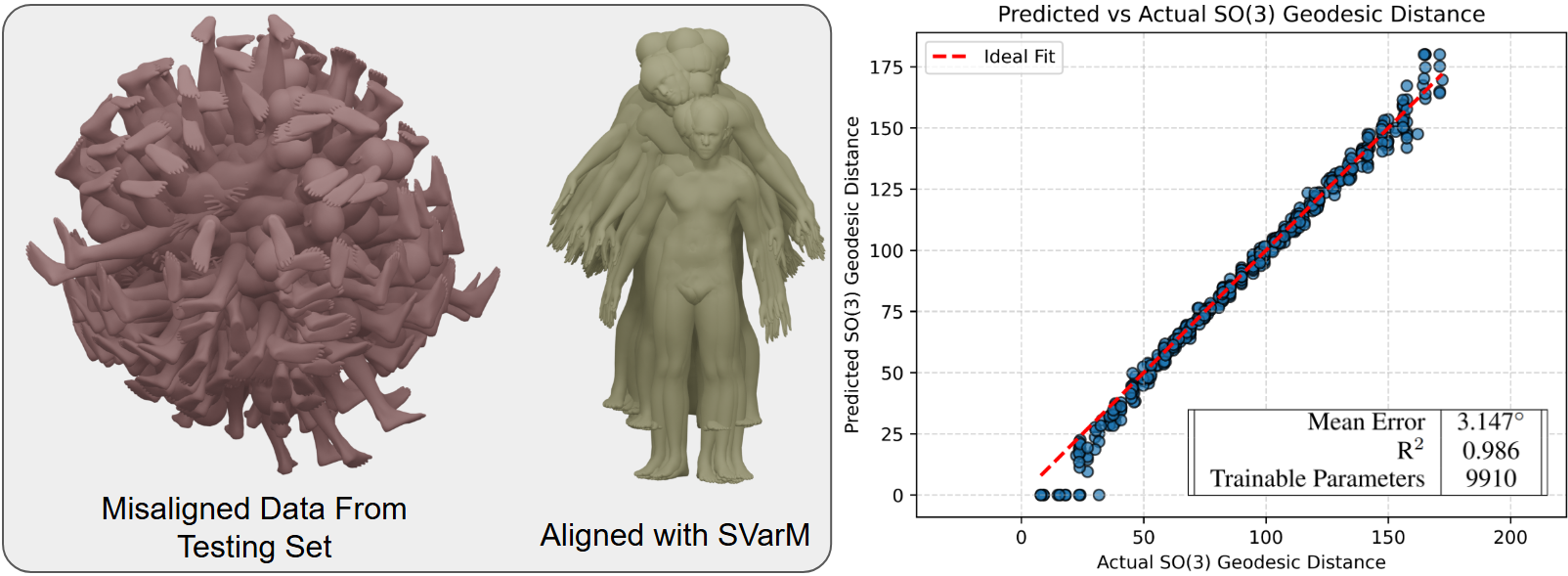}
      \caption{Multidimensional Regression to predict rotation matrices. We display a selection of misaligned data (red) and present same data rotated by the matrix predicted by the SVarM regression model (yellow). Additionally, we plot the geodesic distance in $SO(3)$ from the identity for the predicted matrices against the ground truth  matrices (right) reporting the mean error and correlation coefficient on the entire test set.}
      \label{fig:rotation_regression_SO3} 
  %\hfill
  %\begin{minipage}[b]{0.45\textwidth}
  %  \centering
  %      \begin{tabular}{|| r | c ||} 
  %      \hline
  %      Mean Error & 3.147$^\circ$\\ 
  %      R$^2$& 0.986\\
  %      Trainable Parameters&9910\\
  %      \hline
  %      \end{tabular}
  %\end{minipage}
\end{figure}
The training procedure mirrors our previous setup: we generate synthetic data by applying random 3D rotations to FAUST meshes and use the known rotation matrices as supervision. This results in a set of 12840 meshes and train the model on 80\% of the data and test it on the remaining 20\%. In Figure~\ref{fig:rotation_regression_SO3}, we display test examples of meshes before and after applying the predicted rotation matrix. In addition, we report quantitative metrics, including the average error and the coefficient of determination for the geodesic distances in $SO(3)$ of the rotation matrices. This experiment demonstrates that our regression framework can be extended from constrained single-axis alignment to general 3D rotational correction, paving the way for integration into more complex pipelines involving arbitrary pose variations and sensor noise.

To further highlight our method's effectiveness at solving the rotation alignment problem, we compare with results obtained by directly minimizing, over the rotation group $SO(3)$, the distance from the rotated surface to a target surface in the canonical upright position. Specifically, as a measure of surface discrepancy, we leverage the varifold kernel metric that has been routinely used for such purpose in several past works \citep{kaltenmark2017general,hartman2022elastic}. Note that other choices with similar properties, such as the Chamfer distance from the computer vision literature, could be considered as well. The resulting discrepancy function is then minimized based on the quasi-Newton L-BFGS scheme applied to the 3D rotation parameters. This approach produces a mean error of $8.0558^\circ$ with an $R^2$ value of $0.799$.  Additionally, direct minimization over $SO(3)$ involves no training time, but requires 15-30s per mesh to solve the alignment problem. By comparison, our trained model requires less than 15s to align the entire testing set and produces a mean error of $3.147^\circ$ with an $R^2$ value of $0.986$. Moreover, it involves a total of only 9910 trainable parameters and requires less than 40 minutes in total for training.  

\subsection{Classification of MNIST Digits}\label{ssec:mnist_exp}
\begin{figure}[h!]
  \centering
  \includegraphics[width=1.0\textwidth]{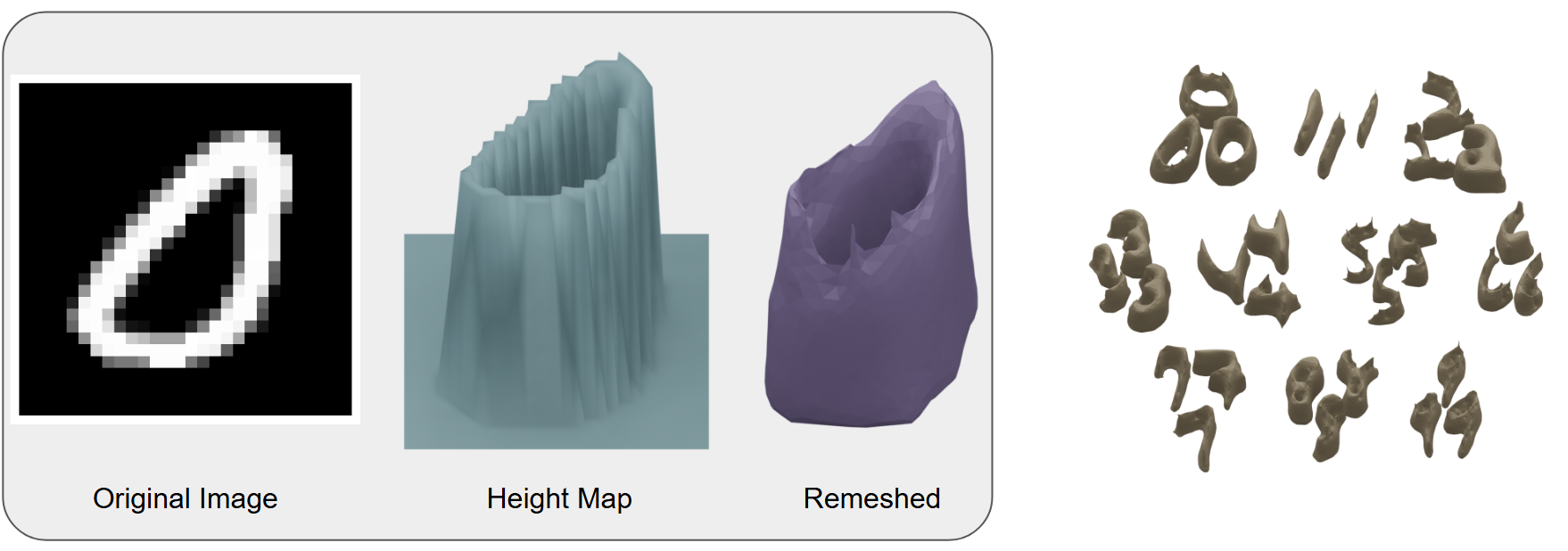} % Replace with your image file name  
\centering
\begin{tabular}{|| r | c c c c | c c ||} 
 \hline
&\multicolumn{4}{c|}{Image Models}&\multicolumn{2}{c||}{3D Models}\\ [0.5ex] 
&Eff. CapsNet&DNN-5&ConvPMM&PMM&PointNet&SVarM\\
%&\cite{Mazzia2021EfficientCapsNet}&\cite{DNN}&\cite{cook2025parametricmatrixmodels}&\cite{cook2025parametricmatrixmodels}&\cite{Charles_PointNet_2017}&\\
 \hline
Accuracy&\textbf{99.84\%}&97.20\%&98.99\%&97.38\%& 90.28\%&\textbf{97.67\%}\\ 
Trainable Parameters&161824&575051&129416&\textbf{4990}&1604883&\textbf{1850}\\
 \hline
\end{tabular}
\caption{MNIST Surface Construction and Results. An image from the MNIST data set is converted to a 3D height map surface and remeshed, with examples of this preprocessing shown for each class of the MNIST data set. We compare several methods operating either on the images or surfaces, reporting accuracy and trainable parameters to highlight our method's efficiency.}\label{fig:mnist}
\end{figure}
We next evaluate SVarM on the task of classifying handwritten digits from the well-known MNIST data set, which contains 70,000 greyscale images of handwritten digits with 60,000 training images and 10,000 images for testing. To adapt this image data set to our framework, following a similar setup to \citep{kostrikov2018surface}, we convert each digit into a 3D surface representation by treating the grayscale pixel values as a height map. We then apply a remeshing operation to generate closed surface meshes corresponding to each digit. A schematic of the conversion process, along with representative examples of the resulting surface meshes, is shown in Figure~\ref{fig:mnist}. This pre-processing is not computationally burdensome and requires $<$20 minutes for the entire MNIST data set.

We compare our model’s performance to that of several baseline methods \citep{Mazzia2021EfficientCapsNet,DNN,cook2025parametricmatrixmodels} that operate directly on the original image data and are designed to be parameter-efficient. While two of these models (Eff. CapsNet~\citep{Mazzia2021EfficientCapsNet} and ConvPMM~\citep{cook2025parametricmatrixmodels}) achieve slightly higher accuracy, they involve significantly more trainable parameters. In contrast, SVarM achieves competitive accuracy while maintaining a much smaller parameter footprint, underscoring its efficiency and suitability for low-resource learning settings.

%\subsection{Classification of SVHN Digits with Contrast Invariant Representations}
%As an additional assessment of the SVarM framework we analyze contrast invariant representations of images from the SVHN data set of color images of digits from street house numbers. When applied to grayscale images, varifolds offer a way to represent the local geometry of an image in a way that is contrast invariant, meaning the representation remains unchanged under monotonic changes in image intensity. In particular, we consider the 2D varifolds associated with each image viewed as a map $I:[0,1]\times[0,1]\to[0,1]$ where the supports are associated with each pixel are given by $\left((x,y), \frac{\nabla I(x,y)}{\|\nabla I(x,y)\|}\right)$ with weights $\begin{cases}
%    \operatorname{vol}_{(x,y)} & \text{ if } \|\nabla I(x,y)\| \neq 0\\
%    0&\text{ otherwise. }
%\end{cases}$
%We train the SVarM model on $80\%$ of the SVHN data and test the model on the remaining $20\%$.

\subsection{Classification of Shape Graphs}\label{sec:numshapegraphs}
To further evaluate the effectiveness of the SVarM model for classification tasks, we train the model on two separate data sets of shape graphs. First, we train SVarM on a data set of shape graphs derived from the CoMA data set \citep{COMA:ECCV18}, which contains 3D surface scans of human faces. This data set is comprised of approximately 21,000 scans of 12 distinct individuals. From each scan, we extract a shape graph by identifying a consistent set of facial landmarks and connecting them via geodesic paths along the surface of the mesh. A schematic illustration of this shape graph construction process is provided in Figure~\ref{fig:face_graphs}. We should note that shape graphs, just as curves, are naturally embedded in the space of varifolds by simple addition in $\mathcal{V}$ of the varifolds associated to each separate branch.
\begin{figure}[h!]
\centering
\begin{minipage}[c]{.5\textwidth}
  \centering
  \includegraphics[width=\textwidth]{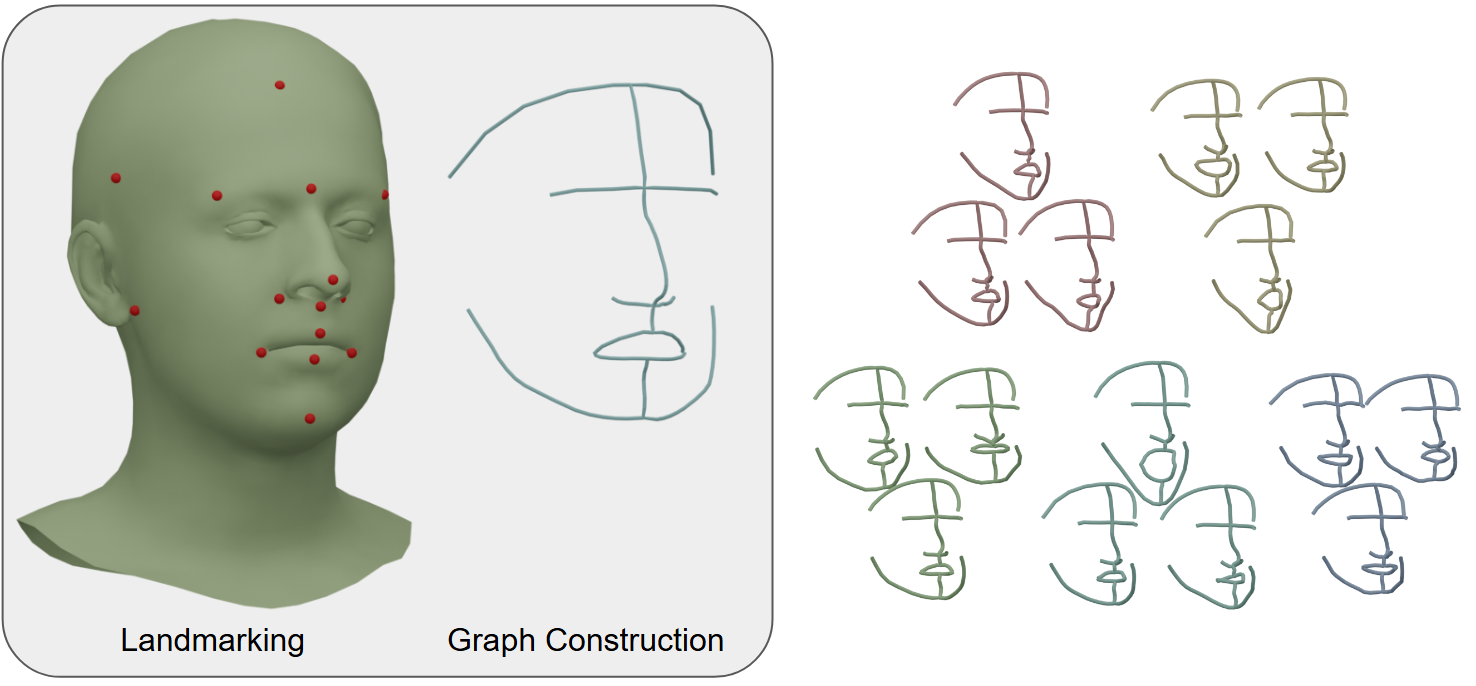} 
\end{minipage}
\begin{minipage}[c]{.45\textwidth}
\begin{tabular}{|| r | c | c ||} 
 \hline
Method & Accuracy & Parameters \\ 
 \hline
PointNet & 64.32\% & 1,605,397 \\
DGCNN& 83.56\% & 1,014,924 \\
VariGrad& 91.28\% & 586,652 \\
SVarM & \textbf{99.98\%} & \textbf{1,980} \\
 \hline
\end{tabular}
\end{minipage}
\caption{COMA Graph Construction and Results. A schematic illustrating the process of extracting a shape graph from a surface in the COMA data set with example graphs from 5 of the 12 classes. We compare the classification accuracy and model size for several models evaluated on this data set.}\label{fig:face_graphs}
\end{figure}
We split the data into training and test sets, using 80\% of the scans for training and the remaining 20\% for testing. To assess performance, we compare our proposed model against three baseline models commonly used for 3D shape classification. Our model achieves significantly higher classification accuracy across all metrics. A quantitative summary of these results is presented in Figure~\ref{fig:face_graphs}.

\begin{figure}[h!]
  \centering
  \includegraphics[width=\textwidth]{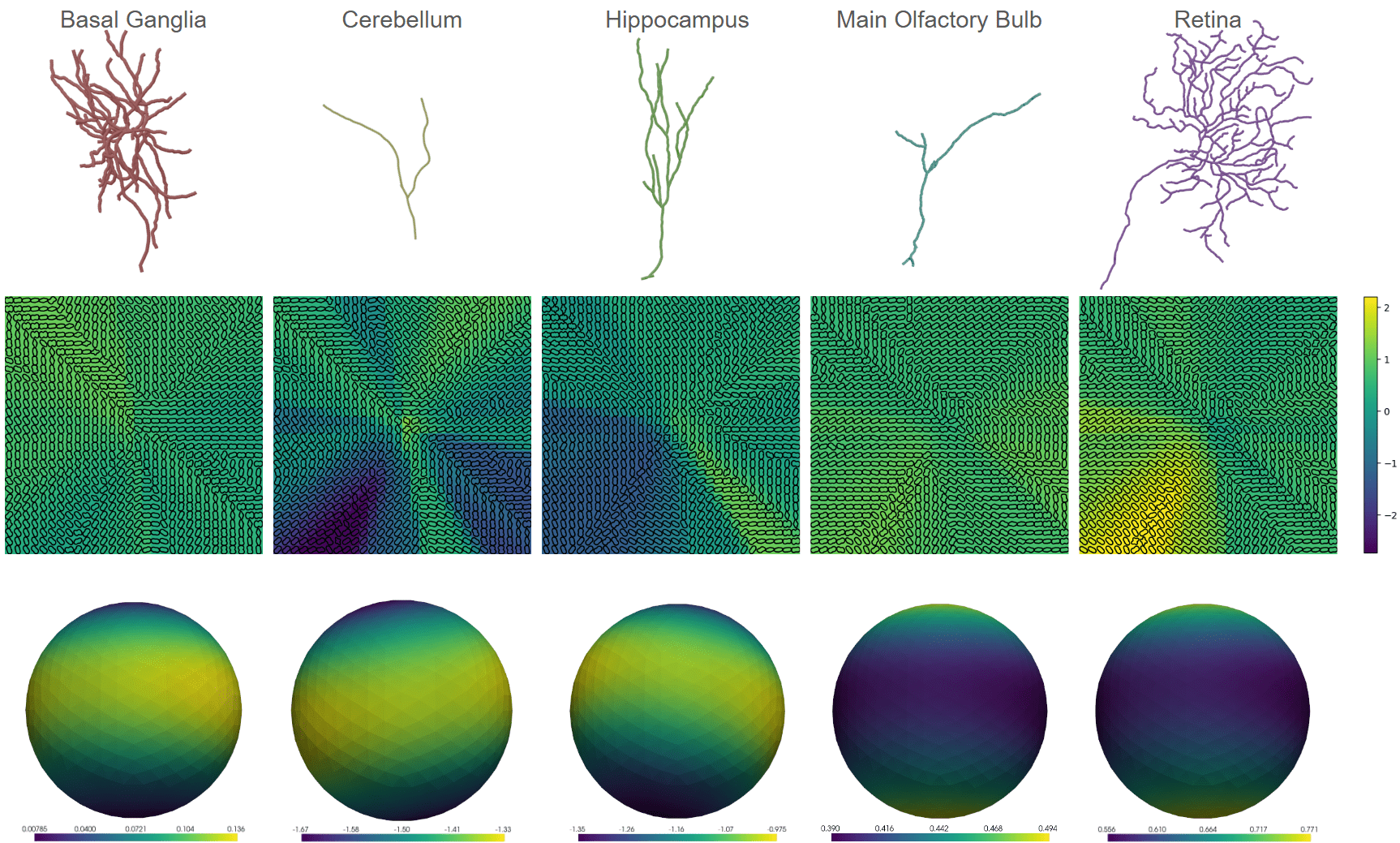} 
  \caption{Neuron Shape Graphs. Row 1: Example of a neuron shape graph from five different regions in the mouse nervous system. Row 2: The marginal spatial profile $H_v(x_1,x_2,0)$ (see Equation~\eqref{eq:profiles}) of the test function of each class is plotted as a color map. The black curves on each pixel illustrate the local direction profile of the function at that point. Row 3: we plot the marginal directional profile $H_x(v)$ of each test function as a color map on $S^2$.}\label{fig:neurons}
\end{figure}
The second data set consists of 3D shape graphs representing neurons from five distinct regions of the mouse nervous system. This data set includes 443 neuron graphs, categorized into the Basal Ganglia, Cerebellum, Hippocampus, Main Olfactory Bulb, and Retina. Representative examples from each class are shown in Figure~\ref{fig:neurons}. We train the model on ~80\% neuron graphs and evaluate its performance on the remaining ~20\%. The trained model correctly classifies 97.73\% of the shape graphs in the testing set. In Figure~\ref{fig:neurons}, for the test functions associated to the different classes, we visualize the marginal spatial and directional profiles defined (respectively) by
\begin{equation}
    H_v(x) := \int_{S^{n-1}} h(x,v) d\mu(x,v)\qquad \text{ and }\qquad H_x(v) := \int_{\R^n} h(x,v) d\mu(x,v). \label{eq:profiles}
\end{equation} 
This experiment highlights the ability of the SVarM model to be trained on smaller data sets without overfitting. In particular, this model is over-parameterized with 1525 trainable parameters and achieves similar accuracy on the training data as on the unseen data.

\subsection{Robustness Studies}
\label{ssec:robustness}
In our final set of experiments, we seek to highlight SVarM's robustness to imaging noise and missing data. We train our classification model on ~80\% of the human face surface data from the CoMA data set. On the testing data, we remove faces from the meshes and report the results on the altered meshes. A visual representation of the meshes with missing faces is shown in Figure \ref{fig:robustness_missing_faces}. We repeat this process, increasing the rate of missing faces and report the classification accuracy in Figure \ref{fig:robustness_missing_faces}. 

\begin{figure}[h!]
\centering
\begin{minipage}[c]{.17\textwidth}
    \centering
    \includegraphics[width=\linewidth]{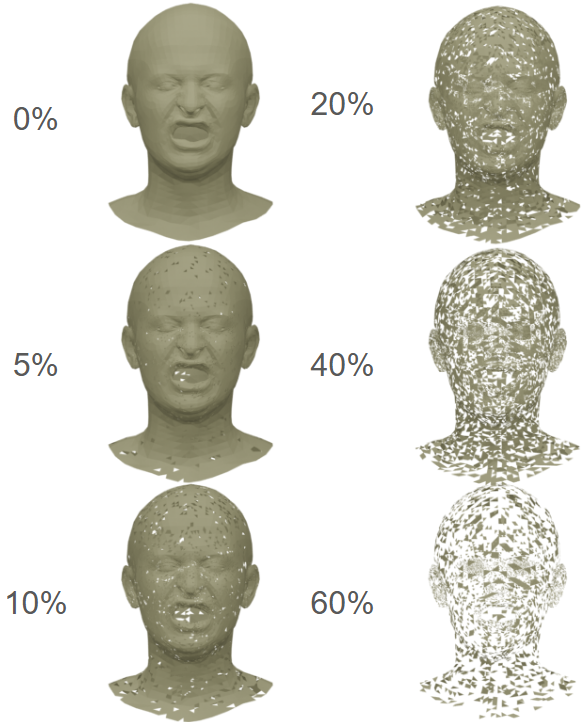}
\end{minipage}
  % Right: Table
\begin{minipage}[c]{.2\textwidth}
    \centering
    \begin{tabular}{|| c | c ||} 
      \hline
      & SVarM \\[0.5ex]
      \hline
       0\% & 99.9\%\\
       5\% & 99.8\%\\ 
      10\% & 98.9\% \\ 
      20\% & 96.3\%\\ 
      40\% & 85.6\% \\ 
      60\% & 67.3\% \\ 
      \hline
    \end{tabular}  
\end{minipage}
\begin{minipage}[c]{.17\textwidth}
    \centering
    \includegraphics[width=\linewidth]{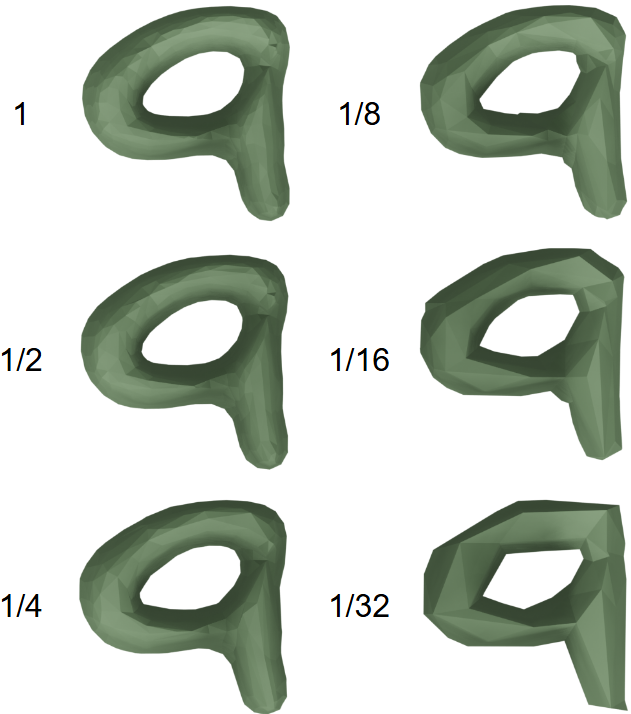}
\end{minipage}
\begin{minipage}[c]{.3\textwidth}
    \centering
    \begin{tabular}{|| c | c c ||} 
      \hline
      & SVarM& PointNet\\[0.5ex]
      \hline
       1 & 97.7\%& 90.3\%\\
       1/2 & 97.7\%& 89.6\%\\ 
       1/4  & 97.4\%& 87.7\%\\ 
       1/8  & 97.3\%& 79.0\%\\ 
       1/16 & 91.5\%& 64.4\%\\ 
       1/32 & 86.2\%& 54.6\%\\ 
      \hline
    \end{tabular}
\end{minipage}
\caption{Robustness experiments under different levels of missing data and mesh decimation. }\label{fig:robustness_missing_faces}
\end{figure}
Furthermore, we consider SVarM's robustness to parameterization by testing the classification model trained on the MNIST digits (c.f. Section \ref{ssec:mnist_exp}) on down-sampled versions of the testing data. For each mesh in the testing set we apply a mesh decimation operation to reduce the number of faces of the mesh to a fraction of the original and report the results on this down-sampled data.  We repeat this process for smaller ratios of total faces and report the classification accuracy in Figure \ref{fig:robustness_missing_faces}. A visual representation of the down-sampled meshes for different ratios of total faces Figure \ref{fig:robustness_missing_faces}.
%\textcolor{red}{Does it make sense to also use rescaling for classification and report the improvement we get doing that? I think it would only be marginally better. If we were presenting the confidence of the model, then there would be a notable difference}

Finally, we explore the robustness of the SVarM model for a regression problem. We consider the trained regression model discussed in Section \ref{ssec:regression_exp}, we apply the face removal process to the testing set and report the results  across varying rates of missing faces. As the model’s predictions are sensitive to the total mass of the input varifolds, its performance deteriorates significantly when a large proportion of faces is removed. To mitigate this issue, we introduce a simple yet effective correction: we normalize the total mass of each varifold to a fixed value. This normalization ensures consistency in the model's input representations, regardless of the number of removed faces. In Figure \ref{fig:regression_robustness}, we illustrate the impact of missing faces on the regression performance, both with and without normalization, demonstrating the robustness of the SVarM regression model, as well as, the improvement gained by re-normalizing the masses.
\begin{figure}[h!]
\centering
    \includegraphics[width=.68\linewidth]{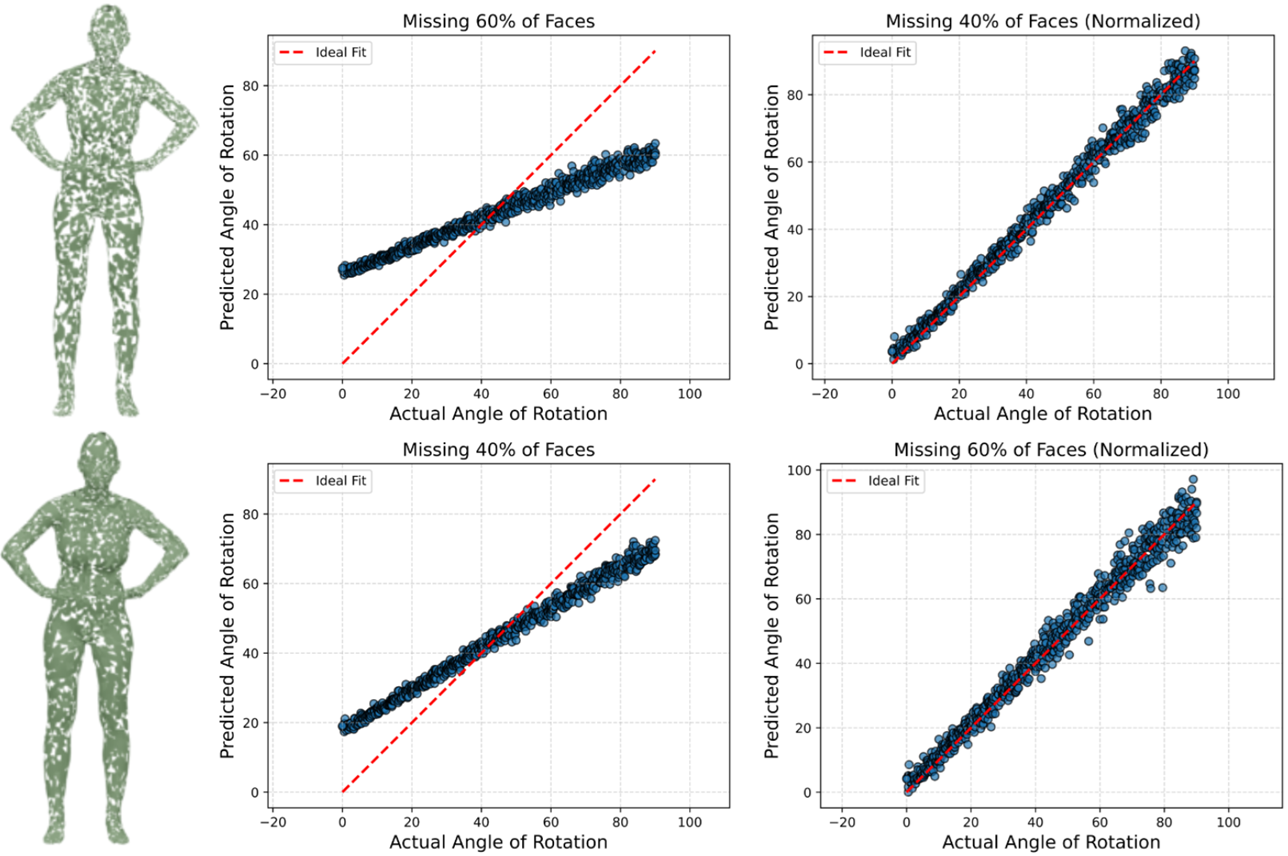}  
    \begin{tabular}{|| c | c  c | c c ||} 
      \hline
      &\multicolumn{2}{c|}{Regression}&\multicolumn{2}{c||}{Regression Rescaled}\\
      \hline
      Missing& Mean Error ($^\circ$) & R$^2$ & Mean Error ($^\circ$) & R$^2$ \\[0.5ex]
      \hline
       0\% &   0.991$^\circ$ &  0.998 & 0.991$^\circ$ & 0.998 \\
       5\% &  1.852$^\circ$ &  0.993 & 1.120$^\circ$ & 0.997 \\ 
      10\% & 2.875$^\circ$ &  0.983 & 1.199$^\circ$ & 0.996 \\ 
      20\% & 5.080$^\circ$ & 0.949 & 1.417$^\circ$ & 0.995 \\ 
      40\% &  9.355$^\circ$ & 0.827 & 1.835$^\circ$ & 0.992 \\ 
      60\% & 13.841$^\circ$ & 0.6245 & 2.582$^\circ$ & 0.984 \\ 
      \hline
    \end{tabular}
  \caption{Robustness experiments for regression on single axis angle of rotation under different proportions of missing faces. }\label{fig:regression_robustness}
\end{figure}

\section{Towards nonlinear regressors and classifiers in $\mathcal{V}$}
In spite of the mathematical and numerical results presented in this paper, one of the obvious limitations of the SVarM model is that it is restricted to regression and classification functionals that are affine on the space of varifolds. Although we showed that these may be enough to separate disjoint sets of shapes under relatively mild assumptions, the counterexample of Figure \ref{fig:counterex_sep} still illustrates that some simple configurations are not separable within the SVarM framework. A natural way of generalizing the approach would be thus to consider nonlinear functionals on $\mathcal{V}$, by analogy with the machine learning models based on neural network representations that are commonly used for data living on Euclidean spaces.

In the context of the infinite-dimensional space of varifolds $\mathcal{V}$, we may build nonlinear maps $\psi: \mathcal{V} \rightarrow \R$ via the composition of multiple affine forms on $\mathcal{V}$ with some nonlinear function on $\R$. More specifically, a simple single hidden layer architecture would consist in modeling $\psi$ as follows:
\begin{equation}
\label{eq:nonlin_Var_fun}
 \psi(\mu) = \sum_{k=1}^{r} c_k \sigma(\langle \mu, h_k \rangle + \beta_k)
\end{equation}
where $h_k \in C_0(\R^n \times S^{n-1})$ are a set of different test functions (each modeled for instance via an MLP as in Section \ref{sec:netw_architecture}), $\beta_k \in \R$ the associated biases, $\sigma:\R \rightarrow \R$ is a nonlinear function (such as RELU, tanh, sigmoid$\ldots$) and $c_k \in \R$ are the coefficients of the output linear function. In the infinite width limit $r \rightarrow \infty$, one can approximate any weak-* continuous function $\mathcal{V} \rightarrow \R$:
\begin{thm}
Let $\sigma: \R \rightarrow \R$ be a continuous non-polynomial function, $C$ a weak-* compact subset of $\mathcal{V}$ and $\psi:C \rightarrow \R$ a weak-* continuous function on $C$. Then for any $\varepsilon >0$, there exist $r \in \mathbb{N}$, $(h_k)_{k=1,\ldots,r}$ in $C_0(\R^n \times S^{n-1})$, $(\beta_k) \in \R^r$ and $(c_k) \in \R^{r}$ such that for all $\mu \in C$: 
\begin{equation*}
\left|\psi(\mu) - \sum_{k=1}^{r} c_k \sigma(\langle \mu, h_k \rangle + \beta_k) \right| \leq \varepsilon
\end{equation*}
\end{thm}
The statement follows directly from the more general result of Theorem 2.1 in \citep{ismailov2024UAT} since the space $\mathcal{V}$ equipped with the weak-* convergence topology is a locally convex space which continuous linear forms are exactly of the form $\mu \mapsto \langle \mu , h \rangle$ (c.f. Proposition \ref{prop:SVarM_affine_maps}). Combined with our earlier remarks in Section \ref{sec:netw_architecture}, this implies that it is possible to approximate any weak-* continuous (nonlinear) function on $\mathcal{V}$ by considering sufficiently many test functions each represented with sufficiently wide (or deep) MLPs.  

Using such more general architectures should theoretically enable to address, in particular, the nonseparability issues of SVarM. For instance, it is easy to derive a function $\psi$ of the form \eqref{eq:nonlin_Var_fun} with only one test function $h$ that can separate the pairs of shapes in the example of Figure \ref{fig:counterex_sep}. Namely, if $h$ is defined so that $\langle \mu,h \rangle$ corresponds to the number of balls on the upper half of the object and $\sigma(x)= 2^{-|x-1|^2}$, then $\psi=1$ for the green surfaces while $\psi=1/2$ for the red ones. We leave it to future work, however, to implement and experiment these types of architectures on shape regression and classification problems. 

\section{Conclusion and future perspectives}
In this work, we propose a learning framework on data sets of curves, surfaces and shape graphs that leverages the (nonlinear) embedding of shapes into varifold space. We focus on this particular measure space as it provides a parameterization-invariant representation of embedded shapes, making the model particularly well-suited for geometric data sets. Analogous to support vector machines—but in the infinite-dimensional setting of varifolds—the SVarM model then learns affine forms $\mu \mapsto \langle \mu,h\rangle + \beta$ to solve regression and classification problems involving 3D geometric data, by approximating a test function $h\in C_0(\R^n\times S^{n-1},\R)$ and bias $\beta$. We derive some theoretical guarantees on SVarM's capacity to approximate arbitrary affine forms, identify conditions under which sets of varifolds or embedded shapes can be linearly separated, and establish stability results for the model. In addition, we present a series of numerical experiments demonstrating the model’s effectiveness in both regression and classification tasks across various geometric data sets. These experiments highlight the model's computational efficiency and its robustness to imaging noise, including missing faces and sampling changes. The results also showcase the fact that SVarM can achieve comparable (and in some cases superior) accuracy to geometric deep learning models but with an overall much smaller number of trainable parameters, making it in line with recent efforts towards developing lightweight deep learning architectures \citep{liu2024lightweight}. 

While we focused on learning affine functions, we also outlined a more general approach to build and learn nonlinear operators over the space of varifolds, which we plan to implement and experiment for various types of supervised or unsupervised problems involving shape data sets in the future. Furthermore, although the present work was primarily concerned with objects such as curves and surfaces, the versatility of the varifold representation makes it possible, in principle, to extend SVarM to other types of geometric data structures, with two cases of interest being high-angular resolution diffusion images \citep{tuch2002high} and spatial transcriptomics data \citep{stouffer2024cross} as these naturally fit the varifold setting. Lastly, an interesting aspect that has not yet been explored in this paper is the use, during training, of regularization penalties on the weights of the neural network defining $h$. Such penalties could enforce certain constraints on $h$ so as to potentially increase the qualitative interpretability of the learned test function.

\section{Ablation Studies}
\subsection{Choice of measure representation}\label{sec:Abl_without_normals}
In this work, we have focused on learning methods that are based on the varifold representations of shapes, i.e. by learning functions on the space $\mathcal{V} = \mathcal{M}(\mathbb{R}^n \times S^{n-1})$, with a shape $q$ being mapped to the varifold $(q, v_q)_* \operatorname{vol}q$. Varifolds combine the information of point position together with the direction of the tangent or normal vector. It is then interesting to compare SVarM against similar approaches but that rely on other possible measure representations. In particular, by marginalizing a varifold with respect to either of the two components $S^{n-1}$ or $\R^n$, one can quantify the advantage of using both the position and direction information for regression and classification problems. We thus compare the SVarM model to these two alternative measure representations, focusing on surface data. First we consider the space of usual \textit{spatial measures}, $\mathcal{M}(\mathbb{R}^n)$, where $q \mapsto q_{*} \operatorname{vol}q$. This representation captures only the distribution of points in space (with respect to the local area measure), discarding all directional information. Conversely, we also consider \textit{area measures} in $\mathcal{M}(S^{n-1})$, where $q \mapsto (v_q)_{*} \operatorname{vol}_q$, which only encodes the distribution of the normal vectors $v_q$ on the sphere independently of the position. Note that area measures obviously do not allow to uniquely identify a shape but are nevertheless objects central to the field of directional statistics \citep{mardia2009directional}.
\begin{figure}[h!]
    \centering
    \begin{tabular}{|| c | c c | c c | c ||} 
      \hline
      &\multicolumn{2}{c|}{1 Axis Rotation Regression}&\multicolumn{2}{c|}{Full Rotation Regression}&MNIST Classification\\
      \hline
      & Mean Error & R$^2$ & Mean Error & R$^2$ & Accuracy \\
      \hline
      $\mathcal{M}(\R^n)$&3.733$^\circ$& 0.967 & 4.302$^\circ$ & 0.954 & 90.69\% \\
      $\mathcal{M}(S^{n-1})$&6.534$^\circ$& 0.908 & 7.847$^\circ$ & 0.895 & 79.86\%\\
      $\mathcal{V}$& 0.991$^\circ$ & 0.998 & 3.147$^\circ$ & 0.986 & 97.67\% \\
      \hline
    \end{tabular}
  \caption{Ablation on Measure Representations of Shapes}\label{tab:ablation_measures}
\end{figure}

Both alternatives encode strictly less geometric information about the shape $q$ compared to the varifold representation. We retrain similar models for each of these measure representations, namely by learning a function on $\R^n$ or $S^{n-1}$ instead of the product space. The results of this comparison are reported in Table~\ref{tab:ablation_measures}. They indicate that neither of these simplified models achieves performance comparable to the full SVarM approach.

\subsection{Width and depth of the MLP representation of $h_\theta$}
We here present a second ablation study related to the architectural choices of the MLP parameterization of the test function $h_\theta$ in the SVarM architecture, specifically on the width and depth of the neural network. These results justify the choice of architecture used in the paper.  
\begin{figure}[h!]
    \centering
    \begin{tabular}{|| c | c c c ||} 
      \hline
      &\multicolumn{3}{c||}{MNIST Classification}\\
      \hline
      & Training Accuracy & Testing Accuracy & Trainable Parameters\\
      \hline
      $\operatorname{mlp}(6,8,32,10)$& 95.51\% & 95.12\% &674\\
      $\operatorname{mlp}(6,16,64,10)$& 98.12\%&97.67\%&1850\\
      $\operatorname{mlp}(6,32,128,10)$& 97.28\% & 97.47\% &5738\\
      \hline
    \end{tabular}
  \caption{Ablation on Width of MLP Parameterizations of Test Functions }\label{tab:ablation_width}
\end{figure}

Increasing the width of the hidden layers generally improves both training and testing accuracy. Notably, the jump from $\operatorname{mlp}(6,8,32,10)$ to $\operatorname{mlp}(6,16,64,10)$ significantly boosts performance, suggesting that a minimum representational capacity is needed for the model to generalize well. However, further increasing width to $\operatorname{mlp}(6,32,128,10)$ yields only marginal improvement in test accuracy, despite more than tripling the number of parameters. This indicates diminishing returns and potential over-parameterization.
\begin{figure}[h!]
    \centering
    \begin{tabular}{|| c | c c c ||} 
      \hline
      &\multicolumn{3}{c||}{MNIST Classification}\\
      \hline
      & Training Accuracy & Testing Accuracy & Trainable Parameters\\
      \hline
      $\operatorname{mlp}(6,64,10)$& 94.75\% &95.39\%&1098\\
      $\operatorname{mlp}(6,16,64,10)$& 98.12\%&97.67\%&1850\\
      $\operatorname{mlp}(6,16,64,64,10)$& 97.64\% & 96.18\% &6010\\
      \hline
    \end{tabular}
  \caption{Ablation on Depth of MLP Parameterizations of Test Functions}\label{tab:ablation_dept}
\end{figure}
Adding more hidden layers beyond a moderate depth appears to reduce generalization performance. While the two-hidden-layer model $\operatorname{mlp}(6,16,64,10)$ achieves the highest test accuracy, the deeper model $\operatorname{mlp}(6,16,64,64,10)$ over-fits slightly, achieving higher training accuracy but worse test accuracy. Conversely, the shallow model $\operatorname{mlp}(6,64,10)$ under-performs.

\section*{Acknowledgments}
This work was supported by the National Science Foundation through the grants DMS-2438562 and DMS-2426550. The authors would also like to thank Murad Hossen for his assistance with the mouse neuron traces dataset used in the numerical experiments section.

\end{document}